\documentclass[letterpaper]{article} 
\usepackage[submission]{aaai24}  
\usepackage{times}  
\usepackage{helvet}  
\usepackage{courier}  
\usepackage[hyphens]{url}  
\usepackage{graphicx} 
\usepackage{natbib}  
\usepackage{caption} 
\usepackage{algorithm}
\usepackage{algorithmicx}
\usepackage[noend]{algpseudocode}
\usepackage[dvipsnames]{xcolor}

\newcommand{\method}{Quilt}

\usepackage{newfloat}
\usepackage{listings}
\DeclareCaptionStyle{ruled}{labelfont=normalfont,labelsep=colon,strut=off} 
\floatstyle{ruled}
\newfloat{listing}{tb}{lst}{}
\floatname{listing}{Listing}
\usepackage{amsmath}
\usepackage{multirow}
\usepackage{amsfonts}
\usepackage{amsthm}

\usepackage{gensymb}

\theoremstyle{definition}
\newtheorem{definition}{Definition}
\newtheorem{theorem}{Theorem}
\newtheorem*{theorem*}{Theorem}

\usepackage{booktabs}       
\usepackage{graphicx, caption, subcaption}
\usepackage{multirow}

\usepackage{makecell}
\usepackage{enumitem}
\usepackage{afterpage}

\title{Quilt: Robust Data Segment Selection against Concept Drifts}
\author{
    Minsu Kim, Seong-Hyeon Hwang, Steven Euijong Whang
}
\affiliations{
    KAIST\\
    \{ms716, sh.hwang, swhang\}@kaist.ac.kr
}

\begin{document}

\maketitle

\begin{abstract}

Continuous machine learning pipelines are common in industrial settings where models are periodically trained on data streams. Unfortunately, {\em concept drifts} may occur in data streams where the joint distribution of the data $X$ and label $y$, $P(X, y)$, changes over time and possibly degrade model accuracy. Existing concept drift adaptation approaches mostly focus on updating the model to the new data possibly using ensemble techniques of previous models and tend to discard the drifted historical data. However, we contend that explicitly utilizing the drifted data together leads to much better model accuracy and propose \method{}, a {\em data-centric} framework for identifying and selecting data segments that maximize model accuracy. To address the potential downside of efficiency, \method{} extends existing data subset selection techniques, which can be used to reduce the training data without compromising model accuracy. These techniques cannot be used as is because they only assume virtual drifts where the posterior probabilities $P(y|X)$ are assumed not to change. In contrast, a key challenge in our setup is to also discard undesirable data segments with concept drifts. \method{} thus discards drifted data segments and selects data segment subsets holistically for accurate and efficient model training. The two operations use gradient-based scores, which have little computation overhead. In our experiments, we show that \method{} outperforms state-of-the-art drift adaptation and data selection baselines on synthetic and real datasets.

\end{abstract}

\section{Introduction}

Robust AI is becoming important especially in continual learning, which is common in industrial settings where models need to be periodically trained on data streams. The applications include manufacturing, meteorology, finance, and more. Here we assume that {\em concept drifts} can occur where the joint distribution of the data $X$ and label $y$, $P(X, y)$ may change over time, leading to decision boundary shifts and model accuracy degradation.

A naive approach of concept drift adaptation is to discard all the historical data, but we would like to retain previous data that is still useful for training. There are many drift adaptation techniques for concept drifts, but most of them take a model-centric approach where they assume that any drifted data is either discarded or replaced with trained models and focus on updating the model or taking an ensemble of previous models to be accurate on the new data. However, throwing away old data is often unacceptable due to heavy investments in data labeling. Even if the knowledge of the data is preserved in the form of models, there are limitations in how much knowledge they retain for making accurate predictions on the new data.

Instead, we contend that taking a {\em data-centric} approach of selecting data segments to train a model can be a more fundamental solution. We assume that any drift detection technique can be used to identify drift points in the data and thus divide the data into {\em data segments} using these points. We then formulate the problem of efficiently selecting which data segments result in the best model accuracy when combined with the current (i.e., newest) data segment.

A key challenge for data-centric approaches is efficiency, and we utilize a recent line of data subset selection techniques\,\cite{DBLP:conf/aaai/KillamsettySRI21,DBLP:conf/icml/KillamsettySRDI21} where the goal is to select a minimal subset of the training data using a validation set for training efficiency while obtaining a similar model accuracy as when training on the entire data. However, these techniques assume virtual drifts where the posterior distribution $P(y|X)$ (i.e., the decision boundary) does not change. The assumption makes sense in this problem because any data can be selected to possibly improve model accuracy, and it is a matter of which data is more useful. In comparison, a concept drift setup assumes that the decision boundary may change, which means that some data may negatively affect model accuracy. As a result, we need to solve the more general problem of performing data segment subset selection while discarding data segments with concept drifts.

We then propose a robust data segment selection framework against concept drifts called \method{} for the purpose of improving model accuracy on recent data. \method{} iteratively performs concept drift detection using conventional detection methods and also selects data segments for training the model if a drift occurs. When selecting data segments, \method{} discards data segments with concept drifts using a {\em disparity} score and also selects a minimal subset of segments without concept drifts such that the model performance is not sacrificed using a {\em gain} score. Both scores can be computed using gradient values on the training and validation sets with little overhead in computation. We also provide theoretical evidence on why the disparity score is effective.

In our experiments, \method{} performs better overall than state-of-the-art drift adaptation baselines on synthetic and real datasets. The benefits are mainly from effectively utilizing previous data segments. In addition, \method{} outperforms existing data-centric concept drift adaptation techniques\,\cite{DBLP:journals/ijon/Ramirez-Gallego17,DBLP:journals/tcyb/DongLSLZ22,DBLP:conf/kdd/Fan04} because they do not explicitly evaluate models on selected data segments as \method{} does.

\textbf{Summary of Contributions:} (1) We propose \method{}, a robust data segment selection framework against concept drifts. (2) We design an efficient data subset selection algorithm that holistically selects core data segments while discarding those with concept drifts. (3) We perform extensive experiments on synthetic and real benchmarks and show that \method{} achieves state-of-the-art accuracy and is efficient.

\section{Preliminaries}
\label{sec:background}

\paragraph{Concept Drift}

A concept drift occurs when the statistical properties of a target domain changes arbitrarily\,\cite{DBLP:journals/tkde/LuLDGGZ19}. Suppose we have the time period [0, $t$] and a sequence of samples $S_{0, t} = \{s_0, \ldots, s_t\}$ where each sample $s_i$ consists of features $X_i$ and a label $y_i$. If the distribution of $S_{0, t}$ is represented as $P_{0, t}(X,y)$, a concept drift at timestamp $t+1$ is formally defined as $P_{0, t}(X,y) \neq P_{t+1, \infty}(X,y)$. Here, we consider $S_{0, t}$ as a previous segment and arriving samples from time $t+1$ as a current segment until the next drift occurs. Note that we may need to look at samples beyond $t+1$ to actually detect the drift.

The joint distribution can be decomposed into a prior distribution and posterior distribution as follows: $P_t(X, y) = P_t(X) \cdot P_t(y|X)$. While many works relevant to distribution drifts assume that the posterior distribution stays the same (referred to as a virtual drift), we assume more realistic drifts where the posterior distribution may change (referred to as an actual or concept drift).

\paragraph{Data Segments}

\method{} assumes that the input data stream consists of {\em data segments} $\mathcal{D} = \{d_1, d_2, \ldots, d_N\}$ where each segment represents a concept. To identify a data segment, any concept drift detection technique\,\cite{DBLP:journals/tkde/LuLDGGZ19,DBLP:journals/datamine/WebbHCNP16,DBLP:journals/inffus/KrawczykMGSW17,gama2014survey} can be used. If two concept drifts are detected at $t_1$ and $t_2$, we assume the data within the time interval [$t_1$, $t_2$] has the same concept and forms a data segment. While some of the previous data segments may benefit the model accuracy on the newest concept, others may even have a negative impact.

\paragraph{Data Subset Selection}

The goal of data subset selection is efficient learning by taking a minimal subset of the training data while still obtaining similar model accuracy. A common approach is to perform coreset selection, which selects the weighted subsets of data that estimate certain properties of the full data, such as the loss or gradient. More recently, data subset selection frameworks like GLISTER\,\cite{DBLP:conf/aaai/KillamsettySRI21} and GRAD-MATCH\,\cite{DBLP:conf/icml/KillamsettySRDI21} focus on both efficiency and robustness using a clean held-out validation set. However, most of these works assume that the posterior distribution $P(y|X)$ (i.e., the decision boundary) does not change, which is not true for concept drifts. While \method{} extends these techniques, it solves the more general problem of handling concept drifts.

\section{Problem Definition}
\label{sec:problemdefinition}

Given an input stream of data segments $\mathcal{D} = \{d_1, d_2, \ldots, d_N\}$, our goal is to make accurate predictions on the current (i.e., newest) data segment $d_N$. As with other existing concept drift works, we assume a multi classification setup. We assume the test set $d_T$ is a subset of $d_N$. For the training and validation sets, a simple setup is to use the previous drifted data only. However, if the previous drifted data is significantly different than $d_N$, then \method{} will not be effective in selecting data segments that are most suitable for the current data segment. Hence, we take two portions of $d_N$ excluding the test set $d_T$ and add one ($d_N^T \subseteq d_N - d_T$) to the training set, while using the other ($d_N^V = d_N - d_T - d_N^T$) as the validation set. To ensure that the current data segment is large enough for such constructions, we train a model only after at least a certain number of samples (say 100) have been collected since the last concept drift. We then select a subset of the previous data segments $\{d_1, d_2, \ldots, d_{N-1}\}$ that minimizes a trained model's loss on $d_N^V$ when added to the original training set $d_N^T$. Denoting $L(\theta, S)$ as the loss on a set $S$ using model parameters $\theta$, our problem can be defined as follows:
\[
\underset{S \subseteq \{d_1, d_2, \ldots, d_{N-1}\}}{\arg\min} \   L(\underset{\theta}{\arg\min} \ L(\theta, d_N^T \cup S), d_N^V)
\]

\begin{figure}[t]
\centering
  \includegraphics[width=\columnwidth]{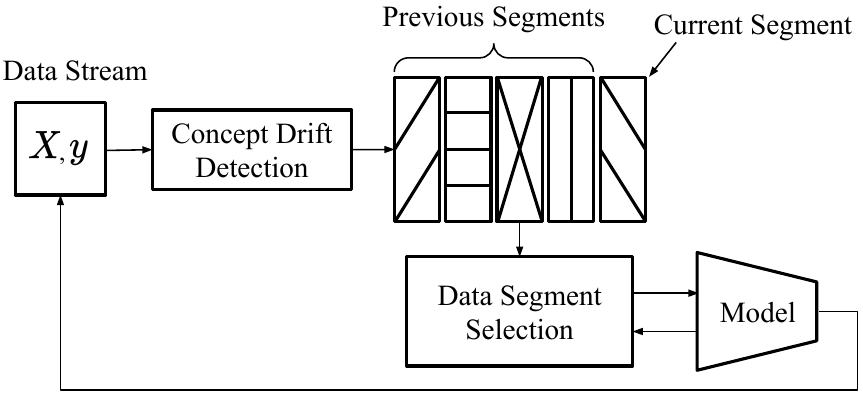}
  \caption{The workflow of \method{}.}
\label{fig:framework}
\end{figure}

\section{\method{} Overview}
\label{sec:overview}

We describe the overall process of \method{} in Figure~\ref{fig:framework} (the full algorithm is in the appendix). For each sample from the input data stream, the concept drift detection component checks if there is a concept drift. There are many existing drift detection methods that use the change of data distribution or model performance, and one can plugin any one of them. If there is a drift, a new data segment is created from the drift point and becomes the current data segment. The data segment selection component then selects segments used to update the model (explained in the next section). If there is no drift, the sample is added to the current data segment.

\section{Data Segment Selection}

When selecting data segments in the presence of concept drifts, \method{} performs two key operations in a holistic framework: (1) discard data segments that represent drifted concepts compared to the current data segment and (2) select a core subset of data segments that have not drifted for efficient model training without compromising accuracy. The two operations utilize {\em disparity} and {\em gain} scores, which are based on gradient values on the training and validation sets and thus have little computation overhead. An additional benefit of using such gradient-based scores is that they are agnostic to the data characteristics. In comparison, statistical distance functions like total variation and Kullback-Leibler divergence are known to have limitations on multivariate data with numeric features and do not scale well to large datasets with high dimensions\,\cite{DBLP:journals/kais/GoldenbergW19}. In the next sections, we explain the gradient computation and elaborate on each score.

\subsection{Gradient Computation}

When computing gradients, we assume neural networks that consist of front layers, which transform the input data to significant embedding features, and a last layer that makes the logit outputs of each class. Let $X_i' \in \mathbb{R}^{d'}$ be the embedding feature of the $i$th input data $X_i$ with a hidden layer dimension of $d'$, and $z_i \in \mathbb{R}^{c}$ be the logit outputs computed by $z_i = w \cdot X_i' + b$ using the last layer weights $w \in \mathbb{R}^{d'\times c}$ and bias $b \in \mathbb{R}^{c}$. To convert a logit $z_i$ into a probability vector $\hat{y}_i$, we use a softmax function: $\hat{y}_i = e^{z_i} / \sum_{i=1}^{c} e^{z_i}$. We can also rewrite the model output $\hat{y}_i$ as a function of the model parameters $\theta$ and input data $X_i$ as $\hat{y}_i = f_{\theta}(X_i)$. With the model output $\hat{y}_i$ and truth label $y_i$, we compute cross-entropy loss between them as $L_i = L(y_i,\hat{y}_i) = -\sum_{j=1}^{c} y_{ij} \cdot \log(\hat{y}_{ij})$. We use last layer gradient approximation $g = (\nabla_{b}L, \nabla_{w}L$) where gradients of the front layers are not used\,\cite{DBLP:conf/icml/KatharopoulosF18, DBLP:conf/iclr/AshZK0A20, DBLP:conf/icml/MirzasoleimanBL20}. Using the chain rule, we can compute the gradient of the $i$th sample as follows: $g_i = (\nabla_{b}L_i, \nabla_{w}L_i) = (\hat{y}_i - y_i, (\hat{y}_i - y_i) \cdot X_i')$.

\subsection{Disparity Score}

We propose a disparity score (abbreviated as $\mathcal{D}$) that measures the dissimilarity between two data distributions and can be used to discard data segments that have concept drifts. Assuming that the data distribution $P(X)$ does not change, a concept drift causes the posterior distribution $P(y|X)$ to change. Hence, a concept drift is similar to how much $y$ changes for the same data. We can capture this notion of disparity in the measure $\mathbb{E}[\|y_t - y_v\|]$, which is the expected amount of label change in a sample where $y_t$ and $y_v$ are truth labels from a training subset and a validation set, respectively. This notion is similar to concept drift severity\,\cite{DBLP:journals/tkde/MinkuWY10}. Directly computing this measure can be expensive where we need to find similar samples in the training and validation sets and measure their label differences. Instead, we define a gradient-based score that is a proxy of this measure and can be computed very efficiently.

\begin{definition}
    The disparity score of a training subset $T$ w.r.t a validation set $V$ is defined as $\mathcal{D}(T, V) = \| \frac{1}{|T|} \sum_{t=1}^{|T|} g_t - \frac{1}{|V|} \sum_{v=1}^{|V|} g_v \| = \| \mathbb{E}[g_t] - \mathbb{E}[g_v] \|$.
\end{definition}

The $\mathcal{D}$ score thus measures the $L_2$-norm distance between two gradient vectors. Intuitively, if a model is fixed, two data segments with similar data distributions should have similar gradients (i.e., low disparity) and vice versa. 

We provide a theoretical justification on why the $\mathcal{D}$ score captures concept drift. For analysis purposes, we make the simplifying assumption that the prior distributions of the training and validation sets are the same, although their labels are different. The proof is in the appendix.

\begin{theorem}
    If training subset $T$ and validation set $V$ have the same prior distribution $P_T(X) = P_V(X)$, but different posterior distributions $P_T(y|X) \neq P_V(y|X)$, then $\mathcal{D}(T,V) \leq \mathbb{E}[\|y_t - y_v\|] \sqrt{1 + \sigma^2}$ where $\sigma = \max(\| \mathbb{E}[X'] \|)$.
\end{theorem}

In practice, the prior distribution may change, but we show in our experiments that the $\mathcal{D}$ score is still effective for measuring drifts.

\subsection{Gain Score}

Our data subset selection is based on theoretical foundations of the recent data subset selection literature\,\cite{DBLP:conf/aaai/KillamsettySRI21,DBLP:conf/icml/KillamsettySRDI21}. 
Suppose there exists historical data for training and a validation set.
It is known that selecting a data subset whose inner product of the average gradients on the subset and the validation set (called the {\em gain}) is positive results in a reduction of the model's validation loss at each epoch\,\cite{DBLP:conf/aaai/KillamsettySRI21}.
Intuitively, a gradient vector represents the magnitude and direction of the model parameter updates when performing gradient descent, and it is desirable for the gradients of the training and validation sets to align.
Computing the gradient values can be extended to data segments, and we define the {\em gain} score for data segments (abbreviated as $\mathcal{G}$) as follows:
\begin{definition}
    The gain score of a training subset $T$ w.r.t. a validation set $V$ is $\mathcal{G}(T, V) = \frac{1}{|T|} \sum_{t=1}^{|T|} g_t \cdot \frac{1}{|V|} \sum_{v=1}^{|V|} g_v = \mathbb{E}[g_t] \cdot \mathbb{E}[g_v]$.
\end{definition}

Compared to the disparity score, the gain score is less sensitive to concept drifts as it is only affected by the magnitude of the two gradients and angle between them, and not the label differences.
Unlike existing data subset selection works where the subset size is set in advance (e.g., top-10\% samples), we opt to select all the data segments that have a positive gain score. The reason is that the data segments may contain some levels of concept drift even after discarding the ones with obvious drifts using disparity scores, so we utilize the sign of the gain score to select the useful data segments that can reduce model loss. This issue does not occur if there are no concept drifts, as no data subset is assumed to decrease model performance.

\subsection{Algorithm}

Algorithm~\ref{alg:dss} shows the data segment selection algorithm of \method{}. We first initialize the model parameters (Step 1). For every epoch, we initialize the training subset to an empty set (Step 3). Next, we compute the average gradient on the validation set (Step 4). We then select all previous segments where the gain is positive, and the disparity is sufficiently small (Steps 5--10). Lastly, we add the current training set and finalize the subset for model training (Step 11). We then update the parameters based on the losses and derived gradients of the selected segments (Step 12). After $T$ epochs, we return the final model parameters (Step 13).

Next, we analyze Algorithm~\ref{alg:dss}'s complexity. We denote $N$ as the number of data segments and $|S|$ as the average number of selected data segments during $T$ epochs. Let $F$ be the forward pass complexity of the last layer and $B$ the backward pass complexity of all layers. Given one validation set, the complexity is $O((N+1)FT + |S|BT)$, where the first term is for computing the gradients of the data segments and the validation set, and the second term is for updating the model parameters with the selected segments.

\begin{algorithm}[t]
\textbf{Input}: Previous data segments $D_{prev} = \{d_1, \ldots, d_{N-1}\}$, training set $d_N^T$, validation set $d_N^V$, loss function $L$, learning rate $\eta$, maximum epochs $T$, disparity threshold $T_d$
\begin{algorithmic}[1]
\State Initialize model parameters $\theta_{0}$
\For {epoch $t$ in [1, $\ldots, T$]}
    \State Initialize training subset $S = \emptyset$
    \State $g_V = \frac{1}{|d_N^V|} \sum_{j=1}^{|d_N^V|} g_j$
    \For {segment $d$ in $D_{prev}$}
        \State $g_d = \frac{1}{|d|} \sum_{k=1}^{|d|} g_k$
        \State $\mathcal{G}_d = g_d \cdot g_V$
        \State $\mathcal{D}_d = \| g_d - g_V \|$
        \If{$\mathcal{G}_d > 0$ and $\mathcal{D}_d < T_d$}
            \State $S = S \cup d$
        \EndIf
    \EndFor
    \State $S = S \cup d_N^T$
    \State Update $\theta_{t} = \theta_{t-1} - \eta \frac{1}{|S|} \sum_{e \in S} \nabla_\theta L_e$
\EndFor
\State \textbf{return} final model parameters $\theta_{T}$
\caption{Data Segment Selection algorithm}
\label{alg:dss}
\end{algorithmic}
\end{algorithm}

\section{Case Study: Two Concepts}

In this section, we analyze the behavior of disparity and gain scores with a simple setup where there are only two concepts 0 and 1, and the samples with these concepts have labels 0 and 1, respectively. Let us say the current data segment $V$ has the concept 1, and there is one previous data segment $T$ with a concept of either 0 or 1. We assume that $P_T(X) = P_V(X)$ for simplicity. Since we use the same model state to compute the gradients of $T$ and $V$, the same prior distribution also leads to the same expected model output $\mathbb{E}[\hat{y}_t] = \mathbb{E}[\hat{y}_v] = (s_1, s_2, \ldots, s_c)$.

\paragraph{Case 1} If $T$'s concept is 1 (i.e., there is no concept drift), we can show that $\mathcal{D}(T, V) \leq 0$ and $\mathcal{G}(T, V) \geq s_1^2 + (s_2 - 1)^2 + \cdots + s_c^2$. The derivation is in the appendix. We also generate synthetic data for this scenario, shown in the appendix. We monitor the $\mathcal{D}$ and $\mathcal{G}$ scores while running \method{} and observe trends that are consistent with the derivations above: Figure~\ref{fig:case_study_disp} shows a simulation result where the $\mathcal{D}$ score is empirically close to zero (not exactly zero due to randomness in the data sampling), indicating no concept drift; and Figure~\ref{fig:case_study_gain} shows a $\mathcal{G}$ score that is positive and keeps on decreasing as the model fits the concept.

\begin{figure}[t]
\centering
  \begin{subfigure}{0.46\columnwidth}
     \centering
     \includegraphics[width=\columnwidth]{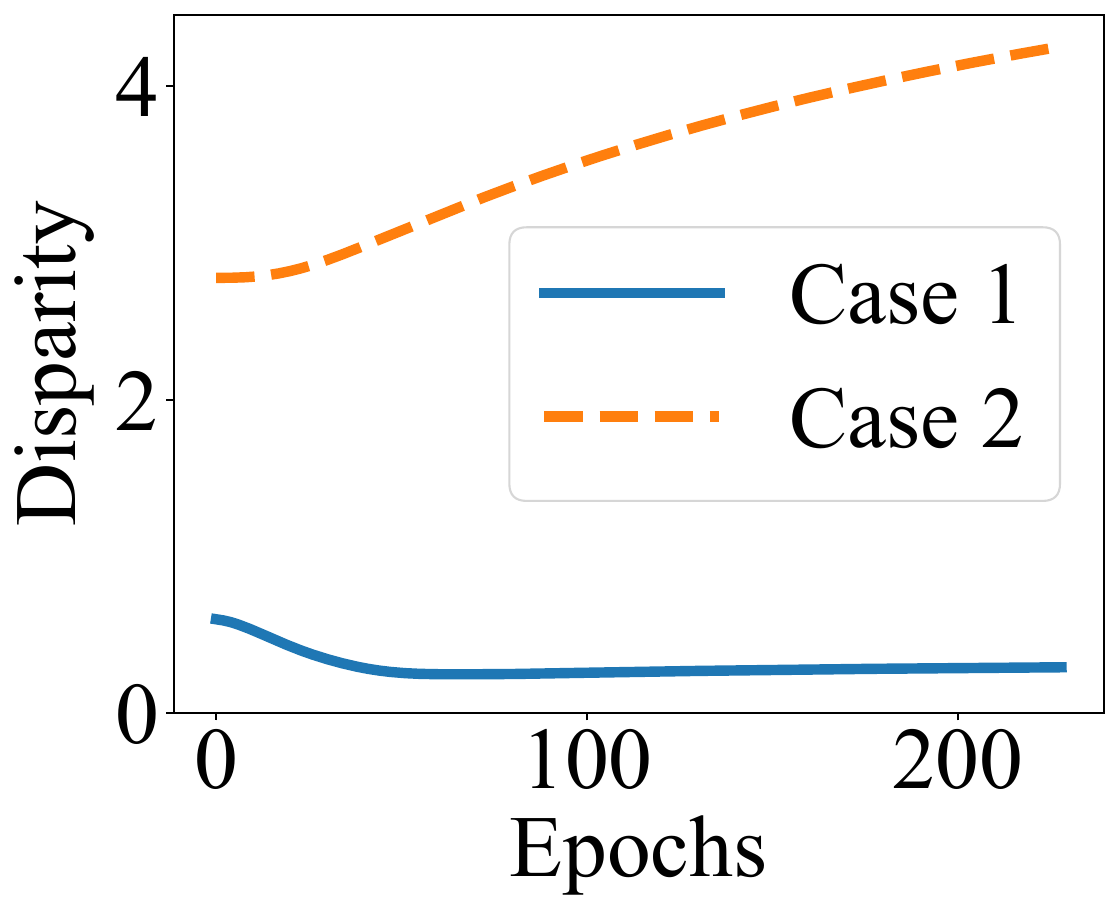}
     \caption{{\sf Disparity}}
     \label{fig:case_study_disp}
  \end{subfigure}
  \begin{subfigure}{0.46\columnwidth}
     \centering
     \includegraphics[width=\columnwidth]{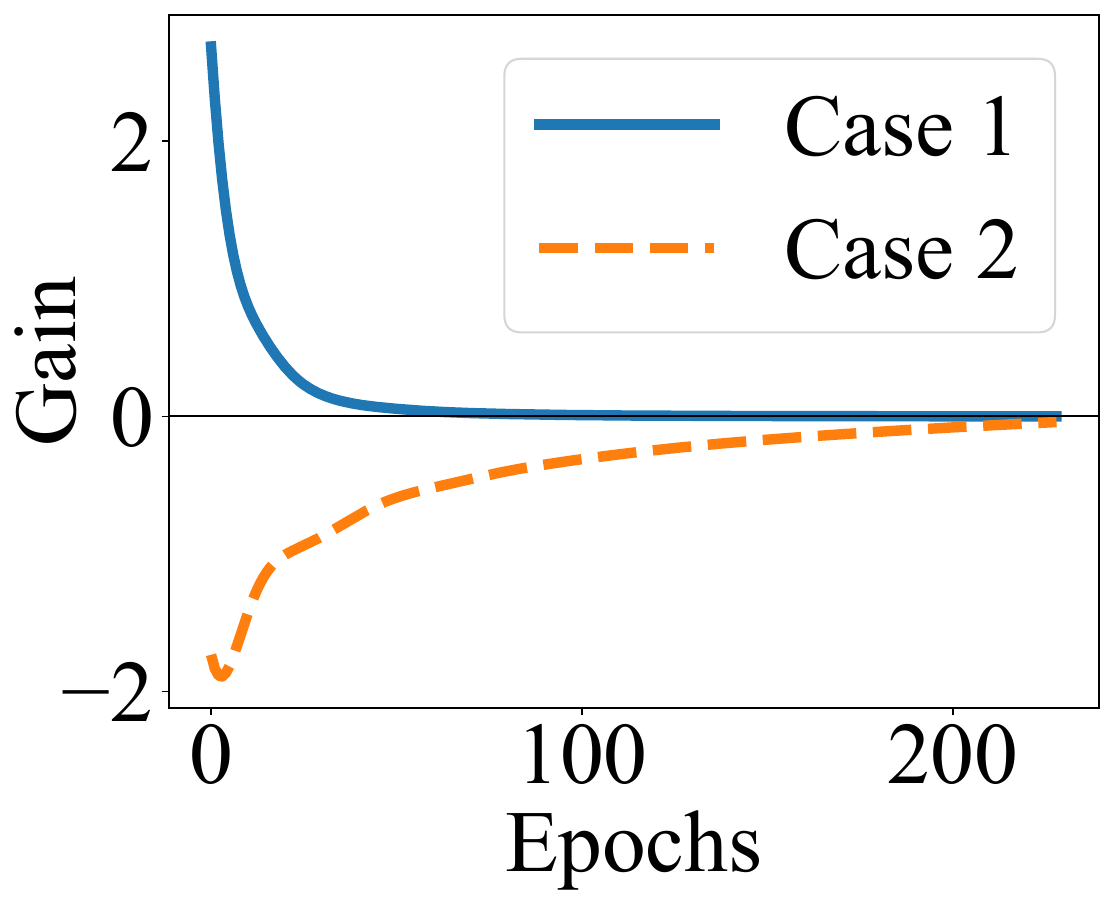}
     \caption{{\sf Gain}}
     \label{fig:case_study_gain}
  \end{subfigure}
  \caption{Disparity ($\mathcal{D}$) and gain ($\mathcal{G}$) scores during model training where Case 1 has no concept drift, and Case 2 does. While $\mathcal{G}$ converges to zero as a converged model has nothing to further gain regardless of $T$'s concept, $\mathcal{D}$ does not as it indicates drift severity.}
\label{fig:casegraph}
\end{figure}

\paragraph{Case 2} If $T$'s concept is 0 (i.e., there is a concept drift), we can show that $\mathcal{D}(T, V) \leq \sqrt{2 (1 + \sigma^2)}$ and $\mathcal{G}(T, V) \geq s_1(s_1-1) + s_2(s_2-1) + \cdots + s_c^{2}$. The derivation is in the appendix. The simulation results in Figure~\ref{fig:case_study_disp} show that the $\mathcal{D}$ score is high compared to non-drift case, indicating a concept drift. In Figure~\ref{fig:case_study_gain}, the $\mathcal{G}$ score is negative and gradually increasing towards zero.

\section{Experiments}
\label{sec:experiments}

We implement \method{} using Python and PyTorch. We evaluate models on separate test sets and repeat all experiments with five different random seeds and write average performances with standard deviations. We use accuracy for evaluation, but also have $F_1$ score results in the appendix. All experiments are run on NVidia Titan RTX GPUs.

\paragraph{Datasets}

We evaluate \method{} on four synthetic and five real datasets. Table~\ref{tbl:datasets} summarizes the datasets with other experimental settings (see more details in the appendix). There are four types of concept drifts depending on how the concept changes: sudden (S), gradual (G), incremental (I), and reoccurring (R). That is, a new concept can suddenly, gradually, or incrementally appear, or an old concept can reoccur. The synthetic datasets are designed to have all the different types of concept drifts\,\cite{DBLP:journals/tkde/LuLDGGZ19}. For some real datasets, there is a mixture of types, which we refer to as ``Complex.''

\begin{table*}[t]
  \caption{For each of the nine datasets (four synthetic and five real datasets), we show the total dataset size, the number of data segments, the number of features, the number of classes, and the drift type.}
  \centering
  \begin{tabular}{cclccccc}
    \toprule
    {\bf Type} & {\bf Dataset} & {\bf Size} & {\bf \#Sgmts} & {\bf \#Ftrs} & {\bf \#Cls} & {\bf Drift Type} \\
    \midrule
    \multirow{4}{*}{Synthetic} & {\sf SEA}\,\cite{DBLP:conf/kdd/StreetK01} & 16K & 8 & 3 & 2 & S, R  \\
    & {\sf Hyperplane}\,\cite{DBLP:conf/kdd/HultenSD01} & 16K & 8 & 10 & 2 & G, I \\
    & {\sf Random RBF}\,\cite{DBLP:journals/jmlr/BifetHKP10} & 16K & 8 & 10 & 5 & S, G, I \\
    & {\sf Sine}\,\cite{DBLP:conf/sbia/GamaMCR04} & 16K & 8 & 4 & 2 & S, R \\
    \midrule
    \multirow{4}{*}{Real} & {\sf Electricity}\,\cite{Harries99splice-2comparative} & 43.2K & 10 & 6 & 2 & Complex \\
    & {\sf Weather}\,\cite{DBLP:journals/tnn/ElwellP11} & 18K & 10 & 8 & 2 & Complex \\
   & {\sf Spam}\,\cite{DBLP:journals/kais/KatakisTV10} & 9.3K & 9 & 499 & 2 & G \\
   & {\sf Usenet1 \& 2}\,\cite{DBLP:conf/ecai/KatakisTV08} & 1.5K & 5 & 99 & 2 & S, R \\
    \bottomrule
  \end{tabular}
  \label{tbl:datasets}
\end{table*}

\paragraph{Model Training}

We train a simple neural network classifier with cross-entropy loss and an Adam optimizer for all the experiments. We use a periodic holdout evaluation method where, if a concept drift occurs, we first wait for a certain number of samples to arrive to train the model and then perform evaluation. The holdout number depends on the dataset, and we set it to be 10--20\% of the size of the average segment size of that dataset. For all the datasets, the holdout number ranges from 60 to 430. We then use half of this data and historical data for the training set, another half for the validation set, and the rest for the test set.
For a fair comparison with other baselines that do not use validation sets, we make them use both our training and validation sets for their training sets.

\paragraph{Baselines}

We compare \method{} with four types of baselines:
\begin{itemize}[leftmargin=*]
    \item {\bf Na\"ive methods}: {\em Full Data} uses all the data segments without any selection and {\em Current Segment} only uses the current segment for training. 
    \item {\bf Model-centric methods}: {\em HAT}\,\cite{DBLP:conf/ida/BifetG09} trains a Hoeffding Adaptive Tree classifier on each sample in an online fashion using the entire data; {\em ARF}\,\cite{DBLP:journals/ml/GomesBRBEPHA17} trains an Adaptive Random Forest classifier on each sample using the entire data; {\em Learn++.NSE}\,\cite{DBLP:journals/tnn/ElwellP11} takes an ensemble of models trained on previous data segments and adjusts their weights depending on their losses on the current data segment; and {\em SEGA}\,\cite{DBLP:journals/tnn/SongLLLZ22} ensembles models trained from equal-length segments of historical data that have the minimum kNN based distributional discrepancy with the current data. We also note that {\em HAT} and {\em ARF} have their own mechanisms for detecting concept drifts.
    \item {\bf Data-centric method}: {\em CVDTE}\,\cite{DBLP:conf/kdd/Fan04} trains a Cross-Validation Decision Tree Ensemble classifier using individual samples that do not have conflicting predictions between shifted decision boundaries due to concept drifts.
    \item {\bf Data Subset Selection methods}: {\em GLISTER}\,\cite{DBLP:conf/aaai/KillamsettySRI21} and {\em GRAD-MATCH}\,\cite{DBLP:conf/icml/KillamsettySRDI21} both train a neural network classifier based on data subset selection methods, but with different criteria. {\em GLISTER} ranks data subsets with gains and selects the top-k subsets with a pre-defined budget. {\em GRAD-MATCH} simultaneously selects data subsets and adjusts their weights to minimize the gradient error. These baselines are not designed to handle concept drifts.
\end{itemize}

\paragraph{Parameters}

For the neural network classifier, we always use one hidden layer with 256 nodes. We set the learning rate to 1e-3 using cross-validation and the number of maximum epochs to 2,000 with early stopping. For each new data segment, we set the disparity threshold using Bayesian optimization with a search interval between (0, 2).

\begin{table*}[t]
  \setlength{\tabcolsep}{4pt}
  \small
  \caption{Accuracy and runtime (sec) results on six datasets. We compare \method{} with all the four types of baselines.} 
  \centering
  \begin{tabular}{l|cccccc|cccccc}
  \toprule
    {Methods} & \multicolumn{2}{c}{\sf SEA} & \multicolumn{2}{c}{\sf Random RBF} & \multicolumn{2}{c|}{\sf Sine} & \multicolumn{2}{c}{\sf Electricity} & \multicolumn{2}{c}{\sf Weather} & \multicolumn{2}{c}{\sf Spam}  \\
    \cmidrule{1-13}
    {} & {Acc.}  & {Time} & {Acc.}  & {Time} & {Acc.}  & {Time} & {Acc.}  & {Time} & {Acc.}  & {Time} & {Acc.}  & {Time} \\
    \midrule
    {Full Data} & {.849}\tiny{$\pm$.005} & {3.36} & {.821}\tiny{$\pm$.007} & {9.43} & {.449}\tiny{$\pm$.032} & {2.25} & {.694}\tiny{$\pm$.010} & {7.42} & {$\mathbf{.800}$}\tiny{$\pm$.005} & {4.33} & {.970}\tiny{$\pm$.003} & {1.17} \\
    {Current Seg.} & {.864}\tiny{$\pm$.004} & {0.20} & {.679}\tiny{$\pm$.010} & {0.56} & {.899}\tiny{$\pm$.004} & {0.94} & {.709}\tiny{$\pm$.009} & {0.52} & {.756}\tiny{$\pm$.007} & {0.26} & {.955}\tiny{$\pm$.003} & {0.16} \\
    \cmidrule{1-13}
    {HAT} & {.825}\tiny{$\pm$.008} & {1.38} & {.514}\tiny{$\pm$.010} & {2.35} & {.293}\tiny{$\pm$.023} & {1.67} & {.691}\tiny{$\pm$.021} & {6.43} & {.729}\tiny{$\pm$.011} & {2.10} & {.888}\tiny{$\pm$.009} & {25.67} \\
    {ARF} & {.825}\tiny{$\pm$.008} & {23.49} & {.645}\tiny{$\pm$.029} & {44.64} & {.821}\tiny{$\pm$.050} & {21.40} & {.713}\tiny{$\pm$.011} & {57.36} & {.775}\tiny{$\pm$.008} & {30.51} & {.921}\tiny{$\pm$.012} & {44.83} \\
    {Learn++.NSE} & {.804}\tiny{$\pm$.005} & {6.65} & {.611}\tiny{$\pm$.009} & {5.68} & {.925}\tiny{$\pm$.004} & {5.73} & {.698}\tiny{$\pm$.004} & {17.26} & {.703}\tiny{$\pm$.007} & {7.86} & {.928}\tiny{$\pm$.006} & {3.81} \\
    {SEGA} & {.797}\tiny{$\pm$.000} & {4.37} & {.825}\tiny{$\pm$.000} & {4.47} & {.253}\tiny{$\pm$.000} & {4.35} & {.637}\tiny{$\pm$.000} & {10.26} & {.777}\tiny{$\pm$.000} & {4.11} & {.858}\tiny{$\pm$.000} & {6.67} \\
    \cmidrule{1-13}
    {CVDTE} & {.806}\tiny{$\pm$.018} & {0.02} & {.614}\tiny{$\pm$.015} & {0.05} & {.857}\tiny{$\pm$.004} & {0.02} & {.689}\tiny{$\pm$.010} & {0.04} & {.731}\tiny{$\pm$.009} & {0.03} & {.917}\tiny{$\pm$.009} & {0.12} \\
    \cmidrule{1-13}
    {GLISTER} & {.857}\tiny{$\pm$.008} & {25.89} & {.794}\tiny{$\pm$.014} & {63.73} & {.879}\tiny{$\pm$.013} & {14.93} & {.698}\tiny{$\pm$.014} & {77.46} & {.793}\tiny{$\pm$.010} & {40.79} & {.971}\tiny{$\pm$.005} & {14.52} \\
    {GRAD-MATCH} & {.853}\tiny{$\pm$.009} & {2.13} & {.790}\tiny{$\pm$.013} & {6.66} & {.547}\tiny{$\pm$.084} & {0.80} & {.686}\tiny{$\pm$.015} & {5.97} & {.795}\tiny{$\pm$.008} & {3.51} & {.968}\tiny{$\pm$.004} & {1.13} \\
    \cmidrule{1-13}
    {\bf \method{}} & {$\mathbf{.888}$}\tiny{$\pm$.004} & {2.20} & {$\mathbf{.833}$}\tiny{$\pm$.008} & {3.22} & {$\mathbf{.936}$}\tiny{$\pm$.005} & {4.88} & {$\mathbf{.728}$}\tiny{$\pm$.007} & {5.24} & {.796}\tiny{$\pm$.004} & {2.00} & {$\mathbf{.974}$}\tiny{$\pm$.003} & {2.59} \\
    \bottomrule
  \end{tabular}
  \label{tbl:performance_representative}
\end{table*}

\subsection{Accuracy and Runtime Results}
\label{sec:accuracyandruntime}

For each dataset, we evaluate \method{}'s accuracy and runtime results for each incoming data segment by setting it as the current segment and then take an average of performances on all the segments. We compare \method{} with the other baselines on six of the datasets as shown in Table~\ref{tbl:performance_representative}. The results for the other three datasets are similar and shown in the appendix. Overall, \method{} outperforms all the baselines in terms of accuracy because it effectively utilizes the drifted data. In comparison, {\em Full Data} is forced to use all the drifted data, while {\em Current Segment} cannot utilize any historical data. {\em HAT} performs worse than \method{} because it adaptively learns the recent data without using previous models or data. The three ensemble methods {\em ARF}, {\em Learn++.NSE}, and {\em SEGA} also perform worse. {\em ARF} can lose useful previous knowledge while replacing an obsolete tree for drift adaptation. Although {\em Learn++.NSE} and {\em SEGA} save all or a buffer's worth of past models and uses the current data segment to ensemble them, the models trained from previous data segments have limitations in fitting to the current data segment with simple ensemble techniques. This result shows the advantage of taking a data-centric approach that adaptively selects suitable data segments with explicit model evaluations. {\em CVDTE} performs worse than \method{} because it simply collects samples that do not have conflicting predictions, regardless of whether they actually benefit model accuracy. For the data subset selection baselines, {\em GLISTER}'s sample-based selection shows more robust results than {\em GRAD-MATCH}'s random batch selection, but the computation time is significantly greater.

\subsection{Data Segment Selection Analysis}
\label{sec:selectionanalysis}

We next verify whether our data segment selection algorithm actually finds core data segments and discards drifted data segments. For each incoming data segment, we collect all the data segments $S$ that have been selected during the entire training at least once as they all influence the final trained model. In addition, we construct a gold standard solution $G$ by performing an exhaustive evaluation where we evaluate all possible combinations of data segments and choose the one that results in the highest accuracy on the validation set. We then measure the precision and recall of $S$ against $G$. Finally, we take the average of the precision and recall values for all the segments. We refer to the exhaustive searching method as {\em Best Segments}. 

Table~\ref{tbl:DADSS_table} shows the average precision and recall results on the four synthetic datasets, each of which has 8 data segments. An average recall value of 0.96--1.00 suggests that \method{} selects almost all of the useful data segments. The average precision is between 0.75--0.98 because \method{} selects some more similar data segments as well. Interestingly, we observe that these extra segments sometimes improve the model accuracy on the test set as shown in Figure~\ref{fig:selection_analysis}, which can happen because the gold standards are selected using the validation set results only. Here we compare \method{} with the {\em Full Data}, {\em Current Segment}, and {\em Best Segments}. As a result, \method{} shows competitive results with {\em Best Segments} and sometimes outperforms it. We suspect that \method{}'s training benefits from the extra data segments during some of its epochs, which improves the model generalization. Hence, \method{} actually achieves most of the room for improvements compared to optimal solutions.

\begin{table}[t]
  \setlength{\tabcolsep}{5pt}
  \caption{Comparison of \method{}'s selected data segments against the Best Segments results on the synthetic datasets.}
  \centering
  \begin{tabular}{ccccccccc}
    \toprule
    Metrics & {\sf SEA} & {\sf Hyperplane} & {\sf Random RBF} & {\sf Sine} \\
    \midrule
    Precision & 0.80 & 0.75 & 0.78 & 0.98 \\
    Recall & 0.98 & 1.00 & 0.96 & 1.00 \\
    \bottomrule
  \end{tabular}
  \label{tbl:DADSS_table}
\end{table}

\begin{figure}[t]
  \centering
  \begin{subfigure}{0.95\columnwidth}
     \centering
     \includegraphics[width=\columnwidth]{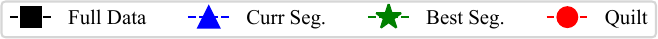}
     \label{fig:legend}
 \end{subfigure}
  \begin{subfigure}{0.461\columnwidth}
     \centering
     \includegraphics[width=\columnwidth]{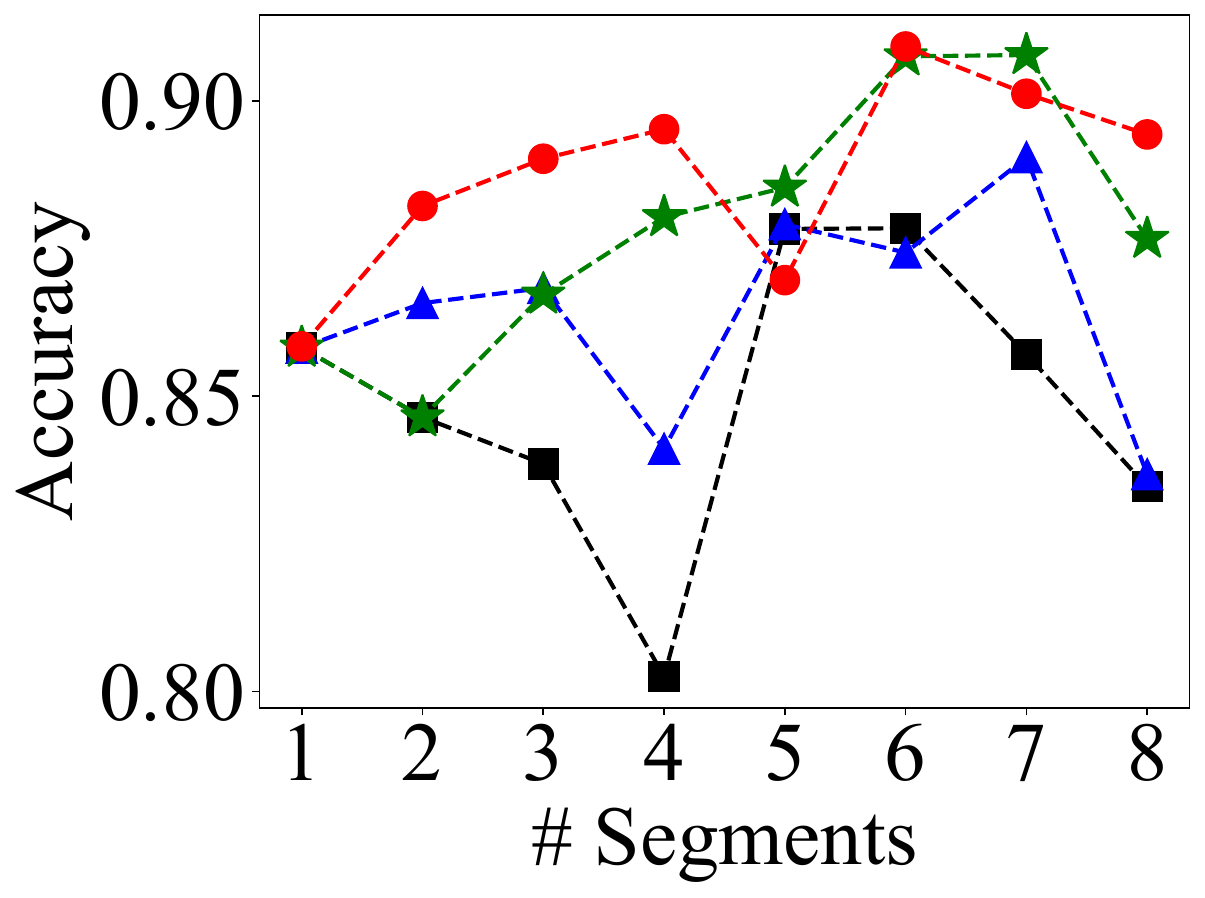}
     \caption{{\sf SEA}}
     \label{fig:SEA}
 \end{subfigure}
 \begin{subfigure}{0.461\columnwidth}
     \centering
     \includegraphics[width=\columnwidth]{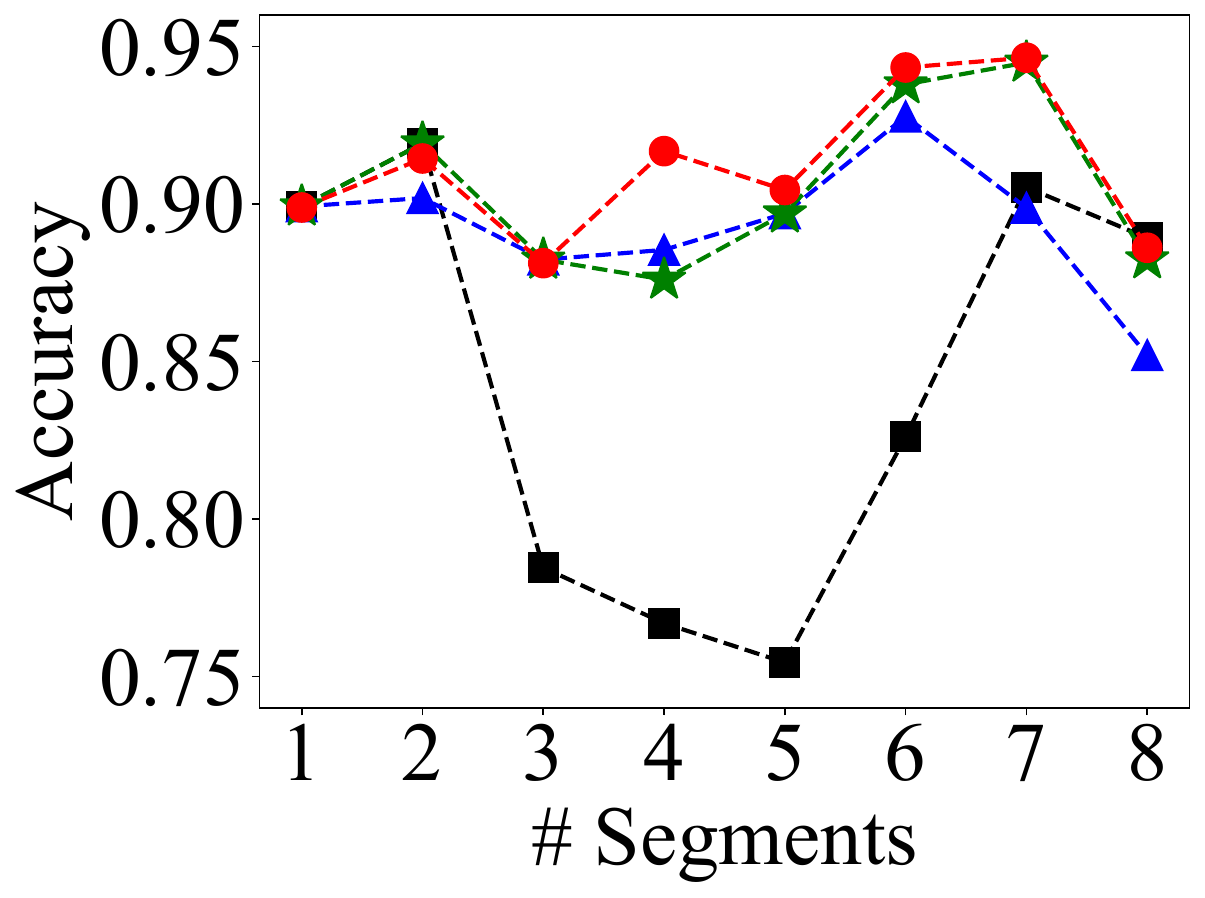}
     \caption{{\sf Hyperplane}}
     \label{fig:Hyperplane}
 \end{subfigure}
 \begin{subfigure}{0.461\columnwidth}
     \centering
     \includegraphics[width=\columnwidth]{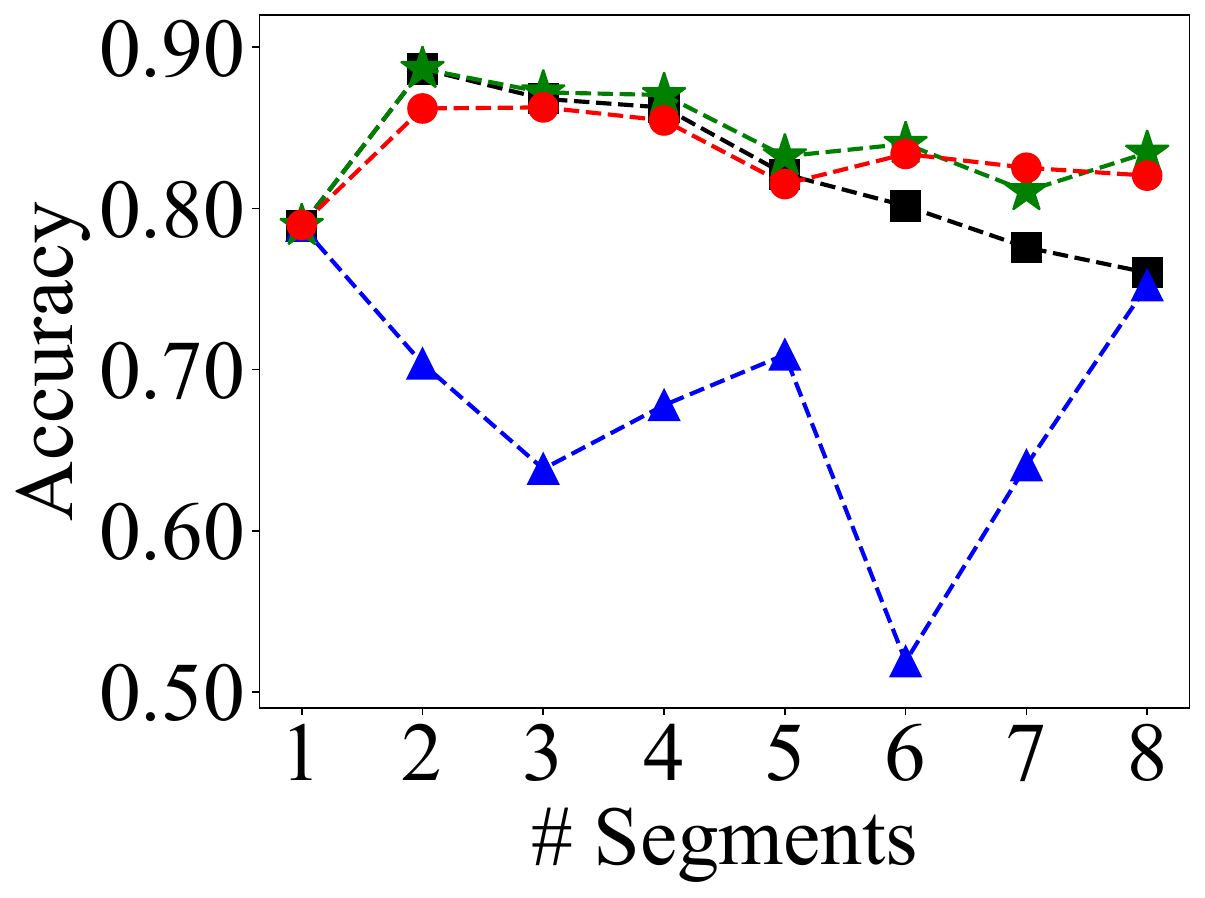}
     \caption{{\sf Random RBF}}
     \label{fig:Random_RBF}
 \end{subfigure}
 \begin{subfigure}{0.461\columnwidth}
     \centering
     \includegraphics[width=\columnwidth]{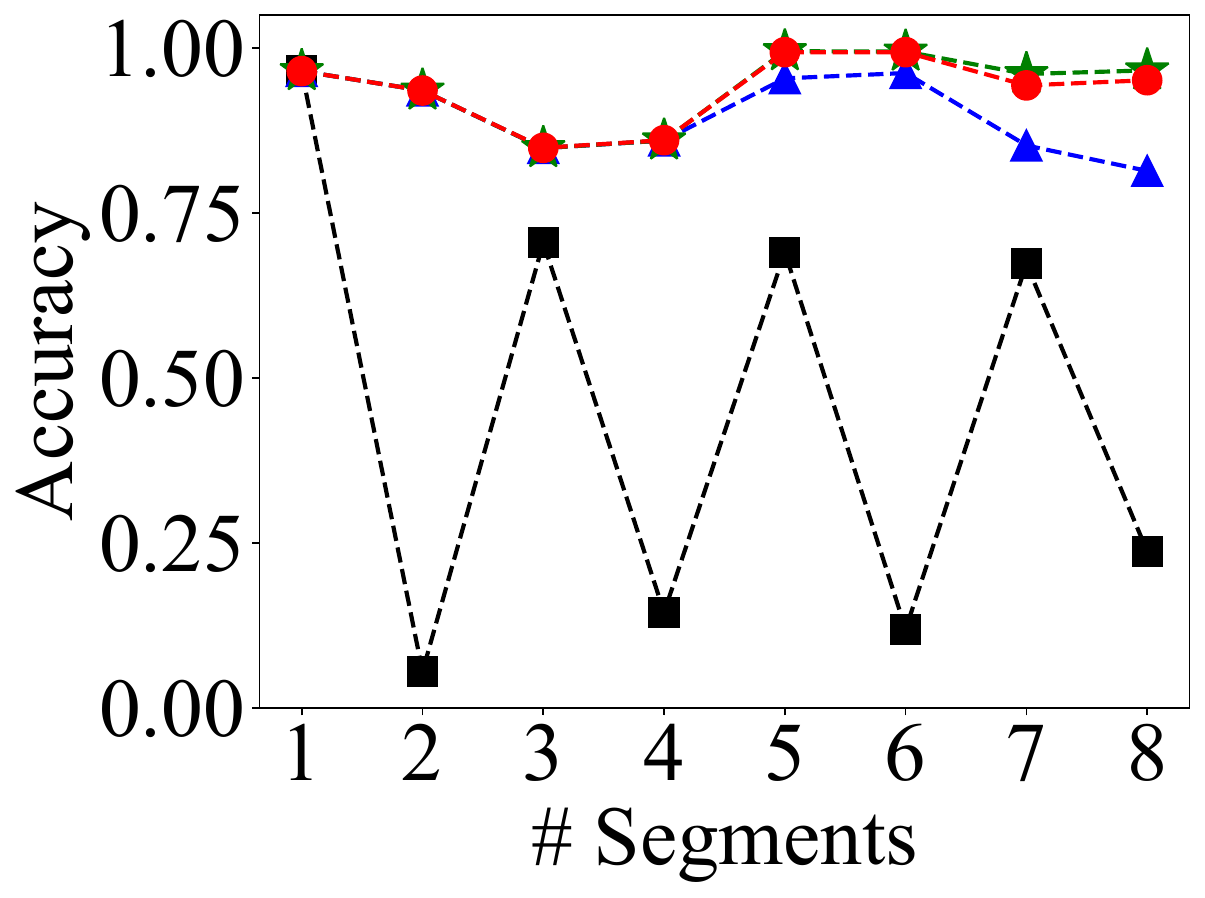}
     \caption{{\sf Sine}}
     \label{fig:Sine}
 \end{subfigure}
 \caption{Accumulative model evaluation results against incoming data segments on the four synthetic datasets.}
 \label{fig:selection_analysis}
\end{figure}

\subsection{Ablation Study}
\label{sec:ablation}

We perform an ablation study of \method{} to see how the two gradient-based scores contribute to the overall performance. Table~\ref{tbl:variations_six} evaluates \method{} variants when not using one or both scores for six datasets. The results for the other datasets are similar and shown in the appendix. We compare the accuracy, runtime, the speedup compared to when not using the scores, and the portion of data segments selected (Usage). As a result, removing the disparity score ($\mathcal{D}$) leads to worse accuracy for datasets with severe drifts as \method{} is not able to discard drifted segments. When removing the gain score ($\mathcal{G}$), the accuracy does not decrease significantly, which is expected, but the runtime worsens because the model training is performed on more data. Removing both scores leads to inaccurate and slow results. For example, on the {\sf Random RBF} dataset, \method{} is comparable or more accurate than any variant while being 5.1x faster by only using 39.4\% of the data segments. While there are some exceptions, \method{} largely provides these benefits for all datasets. Thus, both scores are necessary to obtain both accuracy and efficiency.

\begin{table}[t]
  \setlength{\tabcolsep}{3.6pt}
  \small
  \caption{Accuracy, runtime (sec), speedup, and data segment usage results of \method{} when not using one or both scores.} 
  \centering
  \begin{tabular}{clccc}
  \toprule
    {Datasets} & {Methods} & {Acc.} & {Time (Speedup)} & {Usage} \\
    \midrule
    \multirow{4}{*}{\makecell{\sf SEA}} & {W/o both} & {.850\tiny{$\pm$.005}} & {4.49 (1.0$\times$)} & {100.0\%} \\ 
    & {W/o $\mathcal{D}$} & {.890\tiny{$\pm$.004}} & {2.57 (1.7$\times$)} & {42.2\%} \\
    & {W/o $\mathcal{G}$} & {.881\tiny{$\pm$.007}} & {3.06 (1.5$\times$)} & {49.6\%} \\ 
    & {\method{}} & {.888\tiny{$\pm$.004}} & {2.20 (2.0$\times$)} & {30.8\%} \\
    \midrule
    \multirow{4}{*}{\makecell{{\sf Random}\\{\sf RBF}}} & {W/o both} & {.822\tiny{$\pm$.007}} & {16.51 (1.0$\times$)} & {100.0\%} \\ 
    & {W/o $\mathcal{D}$} & {.828\tiny{$\pm$.008}} & {4.32 (3.8$\times$)} & {45.3\%} \\
    & {W/o $\mathcal{G}$} & {.829\tiny{$\pm$.011}} & {9.69 (1.7$\times$)} & {79.7\%} \\ 
    & {\method{}} & {.833\tiny{$\pm$.008}} & {3.22 (5.1$\times$)} & {39.4\%} \\
    \midrule
    \multirow{4}{*}{\makecell{\sf Sine}} & {W/o both} & {.444\tiny{$\pm$.032}} & {3.43 (1.0$\times$)} & {100.0\%} \\ 
    & {W/o $\mathcal{D}$} & {.890\tiny{$\pm$.015}} & {2.75 (1.2$\times$)} & {36.5\%} \\
    & {W/o $\mathcal{G}$} & {.941\tiny{$\pm$.003}} & {7.77 (0.4$\times$)} & {23.2\%} \\ 
    & {\method{}} & {.936\tiny{$\pm$.005}} & {4.88 (0.7$\times$)} & {21.1\%} \\
    \midrule
    \multirow{4}{*}{\makecell{\sf Electricity}} & {W/o both} & {.696\tiny{$\pm$.009}} & {10.71 (1.0$\times$)} & {100.0\%} \\ 
    & {W/o $\mathcal{D}$} & {.711\tiny{$\pm$.013}} & {6.57 (1.6$\times$)} & {52.9\%} \\
    & {W/o $\mathcal{G}$} & {.723\tiny{$\pm$.009}} & {6.68 (1.6$\times$)} & {39.6\%} \\ 
    & {\method{}} & {.728\tiny{$\pm$.007}} & {5.24 (2.0$\times$)} & {27.9\%} \\
    \midrule
    \multirow{4}{*}{\makecell{\sf Weather}} & {W/o both} & {.798\tiny{$\pm$.006}} & {7.67 (1.0$\times$)} & {100.0\%} \\ 
    & {W/o $\mathcal{D}$} & {.794\tiny{$\pm$.004}} & {2.19 (3.5$\times$)} & {40.9\%} \\
    & {W/o $\mathcal{G}$} & {.800\tiny{$\pm$.006}} & {6.85 (1.1$\times$)} & {75.5\%} \\ 
    & {\method{}} & {.796\tiny{$\pm$.004}} & {2.00 (3.8$\times$)} & {30.7\%} \\
    \midrule
    \multirow{4}{*}{\makecell{\sf Spam}} & {W/o both} & {.970\tiny{$\pm$.003}} & {4.51 (1.0$\times$)} & {100.0\%} \\  
    & {W/o $\mathcal{D}$} & {.973\tiny{$\pm$.004}} & {2.31 (2.0$\times$)} & {52.1\%} \\
    & {W/o $\mathcal{G}$} & {.972\tiny{$\pm$.002}} & {4.02 (1.1$\times$)} & {60.5\%} \\ 
    & {\method{}} & {.974\tiny{$\pm$.003}} & {2.59 (1.7$\times$)} & {39.7\%} \\
  \bottomrule
  \end{tabular}
  \label{tbl:variations_six}
\end{table}

\subsection{Scalability}
\label{sec:scalability}

\begin{figure}[t]
 \centering
 \begin{subfigure}{1\columnwidth}
    \centering
    \includegraphics[width=\columnwidth]{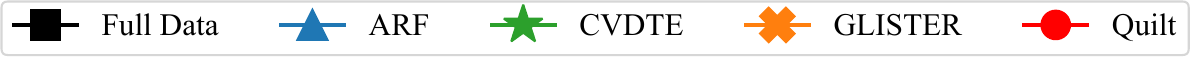}
 \label{fig:legend_scalability}
 \end{subfigure}
 \begin{subfigure}{1\columnwidth}
 \centering
    \includegraphics[width=\columnwidth]{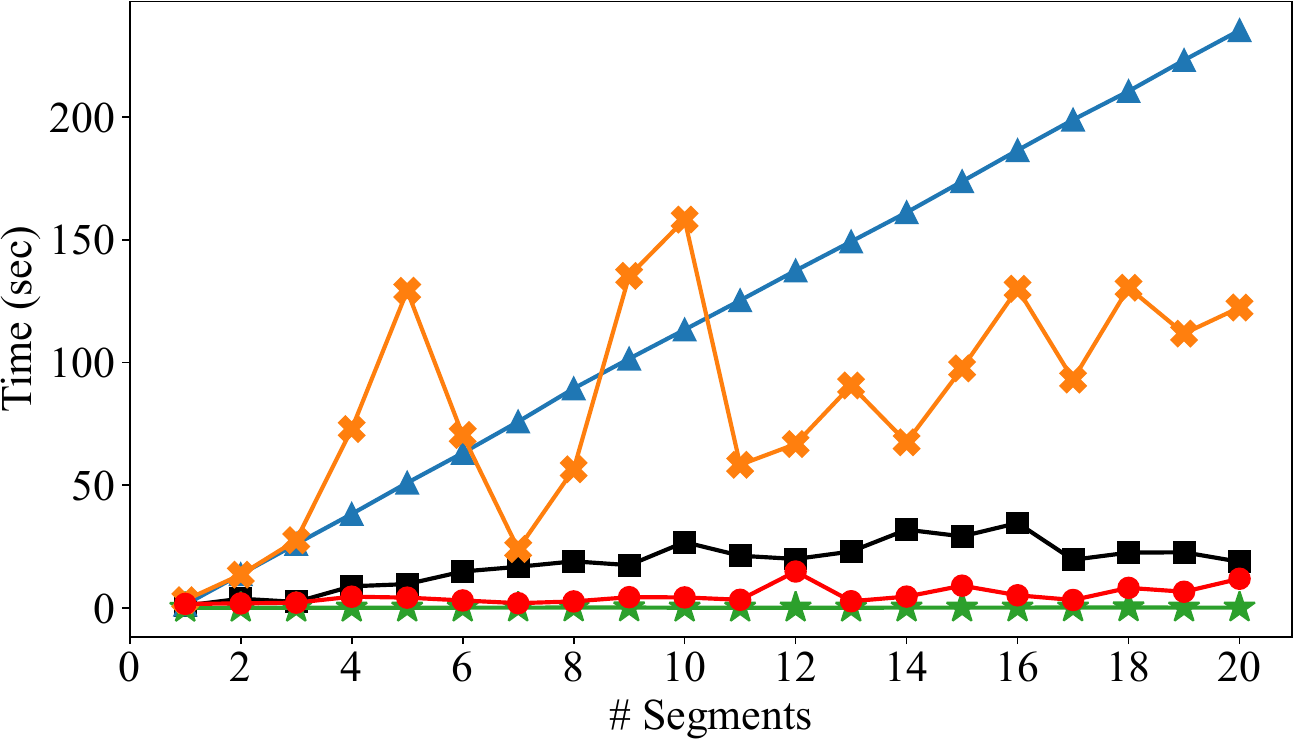}
 \label{fig:scalability_graph}
 \end{subfigure}
 \vspace{-0.4cm}
 \caption{Scalability results of \method{} and baselines on the expanded {\sf Random RBF} dataset.}
 \label{fig:scalability}
\end{figure}

We show the runtimes of \method{} against the number of data segments on the {\sf Random RBF} dataset in Figure~\ref{fig:scalability} where we expand the dataset to have 20 segments from 8. As in the accumulative evaluation setting, we assume incoming data segments and measure the runtime for each data segment when setting it to the current data segment.  As a result, \method{} scales much better than most baselines because it is able to select small core subsets of segments. Another observation is that adding a data segment may sometimes reduce the overall runtime, which can happen if a new data segment has a smaller core subset than that of the preceding segment.

\section{Related Work}
\label{sec:relatedwork}

\paragraph{Drift Detection} Drift detection\,\cite{DBLP:journals/tkde/LuLDGGZ19,DBLP:journals/datamine/WebbHCNP16,DBLP:journals/inffus/KrawczykMGSW17,gama2014survey} techniques have been proposed to quantify concept drifts by identifying change points or change time intervals\,\cite{basseville1993detection}. Drift detection techniques can be largely divided into supervised (e.g., DDM\,\cite{DBLP:conf/sbia/GamaMCR04}, EDDM\,\cite{baena2006early}, and ADWIN\,\cite{bifet2007learning}) and unsupervised (e.g., HDDDM\,\cite{ditzler2011hellinger} and DAWIDD\,\cite{hinder2020towards}) detection techniques depending on whether a trained model is used. In comparison, \method{} is an orthogonal technique where it takes any input of data segments that can be produced from any of these techniques.

\paragraph{Concept Drift Adaptation}

Model-centric concept drift adaptation techniques focus on efficiently updating the given model on new data, and previous drifted data is discarded. We categorize these methods as follows: simple re-training of the model; using ensemble models\,\cite{DBLP:journals/ml/GomesBRBEPHA17,DBLP:journals/tnn/ElwellP11,DBLP:journals/tnn/SongLLLZ22}; and adjusting existing models\,\cite{DBLP:conf/ida/BifetG09,DBLP:conf/kdd/DomingosH00,DBLP:conf/kdd/GamaRM03,DBLP:conf/kdd/HultenSD01} depending on the characteristics of the concept drift. Another categorization\,\cite{DBLP:journals/tnn/ElwellP11} can be done as follows: online versus batch processing depending on how many samples are used for each model training; incremental versus non-incremental based on whether previous data is re-used; and active versus passive whether drift detection is explicitly performed or drifts are assumed to occur at any time. In comparison, \method{} does not discard previous data and carefully utilizes it.

Data-centric approaches for concept drift adaptation have been proposed as well. Data reduction techniques\,\cite{DBLP:journals/ijon/Ramirez-Gallego17} clean the data by deleting redundant or noisy samples and features. Drift understanding techniques\,\cite{DBLP:journals/tcyb/DongLSLZ22} have been proposed where the newest data segment is used as a pattern to filter out obsolete data. 
The filtering is performed based on comparing the cumulative distribution functions of data, and samples that are filtered out are never re-selected even if they could be beneficial later. The most relevant technique to ours is {\em CVDTE}\,\cite{DBLP:conf/kdd/Fan04}, a sample selection technique where given previous and current models, the goal is to select the samples that do not have conflicting predictions. However, all these techniques have the fundamental limitation that there is no way to see if the data preprocessing results actually improve the model accuracy. In comparison, \method{} takes a more data-driven approach of explicitly evaluating models on selected data segments, while minimizing the computation cost using gradient-based disparity and gain scores.

\section{Conclusion}

We proposed \method{}, a data segment selection framework that is robust against concept drifts. Unlike traditional concept drift adaptation frameworks, we take a data-centric approach by directly selecting segments of the training data in a stream that most effectively improves the model accuracy. We design an efficient data subset selection algorithm that selects core data segments while discarding those with concept drifts. The key approach is to use gradient-based disparity and gain scores, which can be used to identify drifts and select useful segments, respectively, and have low computation overheads. Our experiments show how \method{} significantly outperforms both concept drift adaptation and data selection baselines on synthetic and real datasets.

\section*{Acknowledgments}
This work was supported by SK Hynix AICC (K20.05), by the Institute of Information \& Communications Technology Planning \& Evaluation(IITP) grant funded by the Korea government(MSIT) (No.\@ 2022-0-00157, Robust, Fair, Extensible Data-Centric Continual Learning), and by the National Research Foundation of Korea(NRF) grant funded by the Korea government(MSIT) (No.\@ NRF-2022R1A2C2004382). Steven E. Whang is the corresponding author.

\bibliography{main}

\newpage
\onecolumn
\appendix

\section{Appendix -- Techniques}

\subsection{\method{} Algorithm}

We provide the full algorithm of \method{}.

\begin{algorithm}
\begin{algorithmic}[1]

    \Require Data stream $\{e_1(X_1, y_1), e_2(X_2, y_2), \ldots \}$, drift detector $H$, number of samples to wait $n_{wait}$, model function $f$, initial model parameters $\theta_0$, loss function $L$, learning rate $\eta$, maximum epochs $T$, disparity threshold $T_d$

\State Initialize current data segment $d_i = \emptyset$ and index $i = 1$
\State Initialize previous data segments $D_{prev} = \emptyset$
\For {all time $t$}
    \State $\hat{y}_{t} = f_{\theta_t}(e_t)$
    \If{$\hat{y}_{t} == y_{t}$}
        \State $error = 0$
    \Else
        \State $error = 1$
    \EndIf
    \State Add prediction result $error$ to $H$
    \If{drift is detected by $H$}
        \State Add previous segment $d_{i}$ into $D_{prev}$
        \State Gather $n_{wait}$ samples and make current segment $d_{i+1}$
        \State Split current segment $d_{i+1}$ into $d_{i+1}^T$ and $d_{i+1}^V$
        \State $\theta_{t+n_{wait}} = DataSegmentSelection(D_{prev}, d_{i+1}^T, d_{i+1}^V, L, \eta, T, T_d)$
        \State $t = t + n_{wait}$
        \State $i = i + 1$
    \Else
        \State Update $\theta_{t} = \theta_{t-1} - \eta \nabla_\theta L(y_t, \hat{y}_t)$
        \State Add sample $e_t$ into current segment $d_{i}$
        \State $t = t + 1$
    \EndIf
\EndFor
\caption{The \method{} algorithm}
\label{alg:quilt}
\end{algorithmic}
\end{algorithm}

\subsection{Disparity Score Theorem}

We prove the theorem on the upper bound of the disparity score.

\begin{theorem*}
    If training subset $T$ and validation set $V$ have the same prior distribution $P_T(X) = P_V(X)$, but different posterior distributions $P_T(y|X) \neq P_V(y|X)$, then $\mathcal{D}(T,V) \leq \mathbb{E}[\|y_t - y_v\|] \sqrt{1 + \sigma^2}$ where $\sigma = \max(\| \mathbb{E}[X'] \|)$.
\end{theorem*}

\begin{proof}
\begin{equation*}
\begin{split}
\mathcal{D}(T,V) & = \| \mathbb{E}[g_t] - \mathbb{E}[g_v] \| \\
&= \| \mathbb{E}[(\hat{y}_t - y_t, (\hat{y}_t - y_t) \cdot X_t')] - \mathbb{E}[(\hat{y}_v - y_v, (\hat{y}_v - y_v) \cdot X_v')] \| \\
&= \| (\mathbb{E}[\hat{y}_t - y_t], \mathbb{E}[(\hat{y}_t - y_t) \cdot X_t']) - (\mathbb{E}[\hat{y}_v - y_v], \mathbb{E}[(\hat{y}_v - y_v) \cdot X_v']) \| \\
&= \| (\mathbb{E}[\hat{y}_t - y_t] - \mathbb{E}[\hat{y}_v - y_v], \mathbb{E}[(\hat{y}_t - y_t) \cdot X_t'] - \mathbb{E}[(\hat{y}_v - y_v) \cdot X_v']) \| \\
\end{split}
\end{equation*}

Since the $\mathcal{D}$ score measures the $L_2$-norm distance, we can rewrite the function $\mathcal{D}(T,V)$ as follows:

\begin{equation*}
\begin{split}
\mathcal{D}(T,V) &= \{ \| \mathbb{E}[\hat{y}_t - y_t] - \mathbb{E}[\hat{y}_v - y_v] \|^{2} + \| \mathbb{E}[(\hat{y}_t - y_t) \cdot X_t'] - \mathbb{E}[(\hat{y}_v - y_v) \cdot X_v'] \|^{2} \}^{1/2} \\
&= \{ \| \mathbb{E}[(\hat{y}_t - y_t) - (\hat{y}_v - y_v)] \|^{2} + \| \mathbb{E}[(\hat{y}_t - y_t) \cdot X_t' - (\hat{y}_v - y_v) \cdot X_v'] \|^{2} \}^{1/2} \\
\end{split}
\end{equation*}

Next, we subtract and add $(\hat{y}_v - y_v) \cdot X_t'$ once in the second term and get

\begin{equation*}
\begin{split}
\mathcal{D}(T,V) &= \{ \| \mathbb{E}[(\hat{y}_t - y_t) - (\hat{y}_v - y_v)] \|^{2} + \| \mathbb{E}[(\hat{y}_t - y_t) \cdot X_t' - (\hat{y}_v - y_v) \cdot X_t' + (\hat{y}_v - y_v) \cdot X_t' - (\hat{y}_v - y_v) \cdot X_v'] \|^{2} \}^{1/2} \\
&= \{ \| \mathbb{E}[(\hat{y}_t - y_t) - (\hat{y}_v - y_v)] \|^{2} + \| \mathbb{E}[((\hat{y}_t - y_t) - (\hat{y}_v - y_v)) \cdot X_t' + (\hat{y}_v - y_v) \cdot (X_t' - X_v')] \|^{2} \}^{1/2} \\
&= \{ \| \mathbb{E}[(\hat{y}_t - y_t) - (\hat{y}_v - y_v)] \|^{2} + \| \mathbb{E}[((\hat{y}_t - y_t) - (\hat{y}_v - y_v)) \cdot X_t'] + \mathbb{E}[(\hat{y}_v - y_v) \cdot (X_t' - X_v')] \|^{2} \}^{1/2} \\
&= \{ \| \mathbb{E}[(\hat{y}_t - \hat{y}_v) - (y_t - y_v)] \|^{2} + \| \mathbb{E}[((\hat{y}_t - \hat{y}_v) - (y_t - y_v)) \cdot X_t'] + \mathbb{E}[(\hat{y}_v - y_v) \cdot (X_t' - X_v')] \|^{2} \}^{1/2} \\
&= \{ \| \mathbb{E}[\hat{y}_t - \hat{y}_v] - \mathbb{E}[y_t - y_v] \|^{2} + \| \mathbb{E}[(\hat{y}_t - \hat{y}_v) \cdot X_t'] - \mathbb{E}[(y_t - y_v) \cdot X_t'] + \mathbb{E}[(\hat{y}_v - y_v) \cdot (X_t' - X_v')] \|^{2} \}^{1/2} \\
\end{split}
\end{equation*}

Since there is no linear relationship between $\hat{y}_t - \hat{y}_v$ and $X_t'$ (also similar for $\hat{y}_v - y_v$ and $X_t' - X_v'$), we can consider them to be uncorrelated. Being uncorrelated implies that their covariance is zero, which leads to $\mathbb{E}[(\hat{y}_t - \hat{y}_v) \cdot X_t'] = \mathbb{E}[\hat{y}_t - \hat{y}_v] \cdot \mathbb{E}[X_t']$ and $\mathbb{E}[(\hat{y}_v - y_v) \cdot (X_t' - X_v')] = \mathbb{E}[\hat{y}_v - y_v] \cdot \mathbb{E}[X_t' - X_v']$. Then, we have

\begin{equation*}
\begin{split}
\mathcal{D}(T,V) &= \{ \| \mathbb{E}[\hat{y}_t - \hat{y}_v] - \mathbb{E}[y_t - y_v] \|^{2} + \| \mathbb{E}[\hat{y}_t - \hat{y}_v] \cdot \mathbb{E}[X_t'] - \mathbb{E}[(y_t - y_v) \cdot X_t'] + \mathbb{E}[\hat{y}_v - y_v] \cdot \mathbb{E}[X_t' - X_v'] \|^{2} \}^{1/2}
\end{split}
\end{equation*}

We know that the training subset T and validation set V have the same prior distribution, which implies $\mathbb{E}[X_t] = \mathbb{E}[X_v]$. In data subset selection, we compute the gradients of the training subset and validation set with respect to the same model. Then, with the same front layers of the model, the same prior distribution leads to the same expected values of embedding vectors and model outputs, i.e., $\mathbb{E}[X_t'] = \mathbb{E}[X_v']$ and $\mathbb{E}[\hat{y}_t] = \mathbb{E}[\hat{y}_v]$. We can rewrite the equations as follows: $\mathbb{E}[X_t' - X_v'] = \mathbb{E}[X_t'] - \mathbb{E}[X_v'] = 0$ and $\mathbb{E}[\hat{y}_t - \hat{y}_v] = \mathbb{E}[\hat{y}_t] - \mathbb{E}[\hat{y}_v] = 0$. We then obtain:

\begin{equation*}
\begin{split}
\mathcal{D}(T,V) &= \{ \| \mathbb{E}[y_t - y_v] \|^{2} + \| \mathbb{E}[(y_t - y_v) \cdot X_t'] \|^{2} \}^{1/2} \\
\end{split}
\end{equation*}

Since the truth labels and embedding vectors are independent, $\mathbb{E}[(y_t - y_v) \cdot X_t'] = \mathbb{E}[y_t - y_v] \cdot \mathbb{E}[X_t']$. Then, we can write

\begin{equation*}
\begin{split}
\mathcal{D}(T,V) &= \{ \| \mathbb{E}[y_t - y_v] \|^{2} + \| \mathbb{E}[y_t - y_v] \cdot \mathbb{E}[X_t'] \|^{2} \}^{1/2} \\
\end{split}
\end{equation*}

Using the property of norms $\| \mathbb{E}[y_t - y_v] \cdot \mathbb{E}[X_t'] \| \leq \| \mathbb{E}[y_t - y_v] \| \cdot \| \mathbb{E}[X_t'] \|$, we have:

\begin{equation*}
\begin{split}
\mathcal{D}(T,V) &\leq \{ \| \mathbb{E}[y_t - y_v] \|^{2} + \| \mathbb{E}[y_t - y_v] \|^{2} \cdot \| \mathbb{E}[X_t'] \|^{2} \}^{1/2} \\
\end{split}
\end{equation*}

Since the norm function $\|\cdot\|$ is convex, using Jensen's inequality, $\| \mathbb{E}[y_t - y_v] \| \leq \mathbb{E}[\| y_t - y_v \|]$. We finally get

\begin{equation*}
\begin{split}
\mathcal{D}(T,V) &\leq \{ \mathbb{E}[ \| y_t - y_v \| ]^{2} + \mathbb{E}[ \| y_t - y_v \| ]^{2} \cdot \| \mathbb{E}[X_t'] \|^{2} \}^{1/2} \\
&= \{ \mathbb{E}[ \| y_t - y_v \| ]^{2} (1 + \| \mathbb{E}[X_t'] \|^{2}) \}^{1/2} \\
&= \mathbb{E}[ \| y_t - y_v \| ] (1 + \| \mathbb{E}[X_t'] \|^{2})^{1/2} \\
&\leq \mathbb{E}[ \| y_t - y_v \| ] \sqrt{1 + \sigma^2} \\
\end{split}
\end{equation*}

where $\sigma = \max(\| \mathbb{E}[X_t'] \|)$. Therefore, this result shows that the disparity score of a training subset T w.r.t a validation set V is upper bounded by $\mathbb{E}[\|y_t - y_v\|] \sqrt{1 + \sigma^2}$. It shows that the expected values of the predictions on the training and validation sets cancel out, and we are left with the differences between the labels which is a proxy of drift severity.
\end{proof}

\subsection{Gain Score Theorem}

We prove the theorem on a lower bound of the gain score. This theorem is not included in the paper, but used for the derivations in the case study.

\begin{theorem*}
    If training subset $T$ and validation set $V$ have the same prior distribution $P_T(X) = P_V(X)$, but different posterior distributions $P_T(y|X) \neq P_V(y|X)$, then $\mathcal{G}(T,V) \geq \mathbb{E}[\hat{y}_t - y_t] \cdot \mathbb{E}[\hat{y}_v - y_v]$.
\end{theorem*}

\begin{proof}
\begin{equation*}
\begin{split}
\mathcal{G}(T,V) & = \mathbb{E}[g_t] \cdot \mathbb{E}[g_v] \\
&= \mathbb{E}[(\hat{y}_t - y_t, (\hat{y}_t - y_t) \cdot X_t')] \cdot \mathbb{E}[(\hat{y}_v - y_v, (\hat{y}_v - y_v) \cdot X_v')] \\
&= (\mathbb{E}[\hat{y}_t - y_t], \mathbb{E}[(\hat{y}_t - y_t) \cdot X_t']) \cdot (\mathbb{E}[\hat{y}_v - y_v], \mathbb{E}[(\hat{y}_v - y_v) \cdot X_v']) \\
&= \mathbb{E}[\hat{y}_t - y_t] \cdot \mathbb{E}[\hat{y}_v - y_v] + \mathbb{E}[(\hat{y}_t - y_t) \cdot X_t'] \cdot \mathbb{E}[(\hat{y}_v - y_v) \cdot X_v'] \\
\end{split}
\end{equation*}

Since there is no linear relationship between $\hat{y}_t - y_t$ and $X_t'$ (also same for $\hat{y}_v - y_v$ and $X_v'$), we can consider them to be uncorrelated. Being uncorrelated implies that their covariance is zero, which leads to $\mathbb{E}[(\hat{y}_t - y_t) \cdot X_t'] = \mathbb{E}[\hat{y}_t - y_t] \cdot \mathbb{E}[X_t']$ (similarly, $\mathbb{E}[(\hat{y}_v - y_v) \cdot X_v'] = \mathbb{E}[\hat{y}_v - y_v] \cdot \mathbb{E}[X_v']$). Now we have,

\begin{equation*}
\begin{split}
\mathcal{G}(T,V) & = \mathbb{E}[\hat{y}_t - y_t] \cdot \mathbb{E}[\hat{y}_v - y_v] + \mathbb{E}[\hat{y}_t - y_t] \cdot \mathbb{E}[X_t'] \cdot \mathbb{E}[\hat{y}_v - y_v] \cdot \mathbb{E}[X_v'] \\
&= \mathbb{E}[\hat{y}_t - y_t] \cdot \mathbb{E}[\hat{y}_v - y_v] (1 + \mathbb{E}[X_t'] \cdot \mathbb{E}[X_v']) \\
\end{split}
\end{equation*}

Note that the training subset $T$ and validation set $V$ have the same prior distribution, which implies that $\mathbb{E}[X_t] = \mathbb{E}[X_v]$. In data subset selection, we compute the gradients of training subset and validation set with respect to the same model. Then, with the same front layers of the model, the same prior distribution leads to the same expected value of embedding vectors, i.e., $\mathbb{E}[X_t'] = \mathbb{E}[X_v'] = \mathbb{E}[X']$. We thus have the following result:

\begin{equation*}
\begin{split}
\mathcal{G}(T,V) &= \mathbb{E}[\hat{y}_t - y_t] \cdot \mathbb{E}[\hat{y}_v - y_v] (1 + \| \mathbb{E}[X'] \|^2) \\
&\geq \mathbb{E}[\hat{y}_t - y_t] \cdot \mathbb{E}[\hat{y}_v - y_v]
\end{split}
\end{equation*}

This result shows that the gain score of a training subset $T$ w.r.t a validation set $V$ is lower bounded by $\mathbb{E}[\hat{y}_t - y_t] \cdot \mathbb{E}[\hat{y}_v - y_v]$. The lower bound involves the model outputs $\hat{y}_t$ and $\hat{y}_v$, which implies that the gain score is related to the model's state.
\end{proof}

\subsection{Derivations in Case Study}

We prove the results in the case study involving two concepts with and without drift.

\paragraph{Case 1:} $T$'s concept is 1, and $V$'s concept is 1 (i.e., there is no concept drift).

\begin{equation*}
\begin{split}
\mathcal{D}(T,V) &\leq \mathbb{E}[ \| y_t - y_v \| ] \sqrt{1 + \sigma^2} \\
&= \mathbb{E}[ \| (0, 1, \ldots, 0) - (0, 1, \ldots, 0) \| ] \sqrt{1 + \sigma^2} \\
&= \mathbb{E}[ \| (0, 0, \ldots, 0) \| ] \sqrt{1 + \sigma^2} \\
&= 0
\end{split}
\end{equation*}

\begin{equation*}
\begin{split}
\mathcal{G}(T,V) &\geq \mathbb{E}[\hat{y}_t - y_t] \cdot \mathbb{E}[\hat{y}_v - y_v] \\
&= (\mathbb{E}[\hat{y}_t] - \mathbb{E}[y_t]) \cdot (\mathbb{E}[\hat{y}_v] - \mathbb{E}[y_v]) \\
&= (\mathbb{E}[(s_{t1}, s_{t2}, \ldots, s_{tc})] - (0, 1, \ldots, 0)) \\
&\quad \cdot (\mathbb{E}[(s_{v1}, s_{v2}, \ldots, s_{vc})] - (0, 1, \ldots, 0)) \\
&= ((s_1, s_2, \ldots, s_c) - (0, 1, \ldots, 0)) \\ 
&\quad \cdot ((s_1, s_2, \ldots, s_c) - (0, 1, \ldots, 0)) \\
&= (s_1, s_2 - 1, \ldots, s_c) \cdot (s_1, s_2 - 1, \ldots, s_c) \\
&= s_1^2 + (s_2 - 1)^2 + \cdots + s_c^2
\end{split}
\end{equation*}

\paragraph{Case 2: } $T$'s concept is 0, and $V$'s concept is 1 (i.e., there is a concept drift).

\begin{equation*}
\begin{split}
\mathcal{D}(T,V) &\leq \mathbb{E}[ \| y_t - y_v \| ] \sqrt{1 + \sigma^2} \\
&= \mathbb{E}[ \| (1, 0, \ldots, 0) - (0, 1, \ldots, 0) \| ] \sqrt{1 + \sigma^2} \\
&= \mathbb{E}[ \| (1, -1, \ldots , 0) \| ] \sqrt{1 + \sigma^2} \\
&= \sqrt{2 (1 + \sigma^2)}
\end{split}
\end{equation*}

\begin{equation*}
\begin{split}
\mathcal{G}(T,V) &\geq \mathbb{E}[\hat{y}_t - y_t] \cdot \mathbb{E}[\hat{y}_v - y_v] \\
&= (\mathbb{E}[\hat{y}_t] - \mathbb{E}[y_t]) \cdot (\mathbb{E}[\hat{y}_v] - \mathbb{E}[y_v]) \\
&= (\mathbb{E}[(s_{t1}, s_{t2}, \ldots, s_{tc})] - (1, 0, \ldots, 0)) \\
&\quad \cdot (\mathbb{E}[(s_{v1}, s_{v2}, \ldots, s_{vc})] - (0, 1, \ldots, 0)) \\
&= ((s_1, s_2, \ldots, s_c) - (1, 0, \ldots, 0)) \\
&\quad \cdot ((s_1, s_2, \ldots, s_c) - (0, 1, \ldots, 0)) \\
&= (s_1 - 1, s_2, \ldots, s_c) \cdot (s_1, s_2 - 1, \ldots, s_c) \\
&= s_1(s_1-1) + s_2(s_2-1) + \cdots + s_c^{2}
\end{split}
\end{equation*}

During the initial training epochs, the model predictions are almost random, which means that the $s$ values for each class is close to $1/c$. As the model adaptively fits on the new concept, the $s_2$ value gets closer to 1, while the other values get closer to 0. We can theoretically describe the simulation results for $\mathcal{D}$ and $\mathcal{G}$ scores in Figure~\ref{fig:casegraph} using the derivations above. Next, we present the synthetic data generated for the empirical evaluations in Figure~\ref{fig:casedata}. We show the two training subsets for cases 1 and 2 and the validation set.

\begin{figure}[h]
\centering
  \begin{subfigure}{0.32\columnwidth}
     \centering
     \includegraphics[width=\columnwidth]{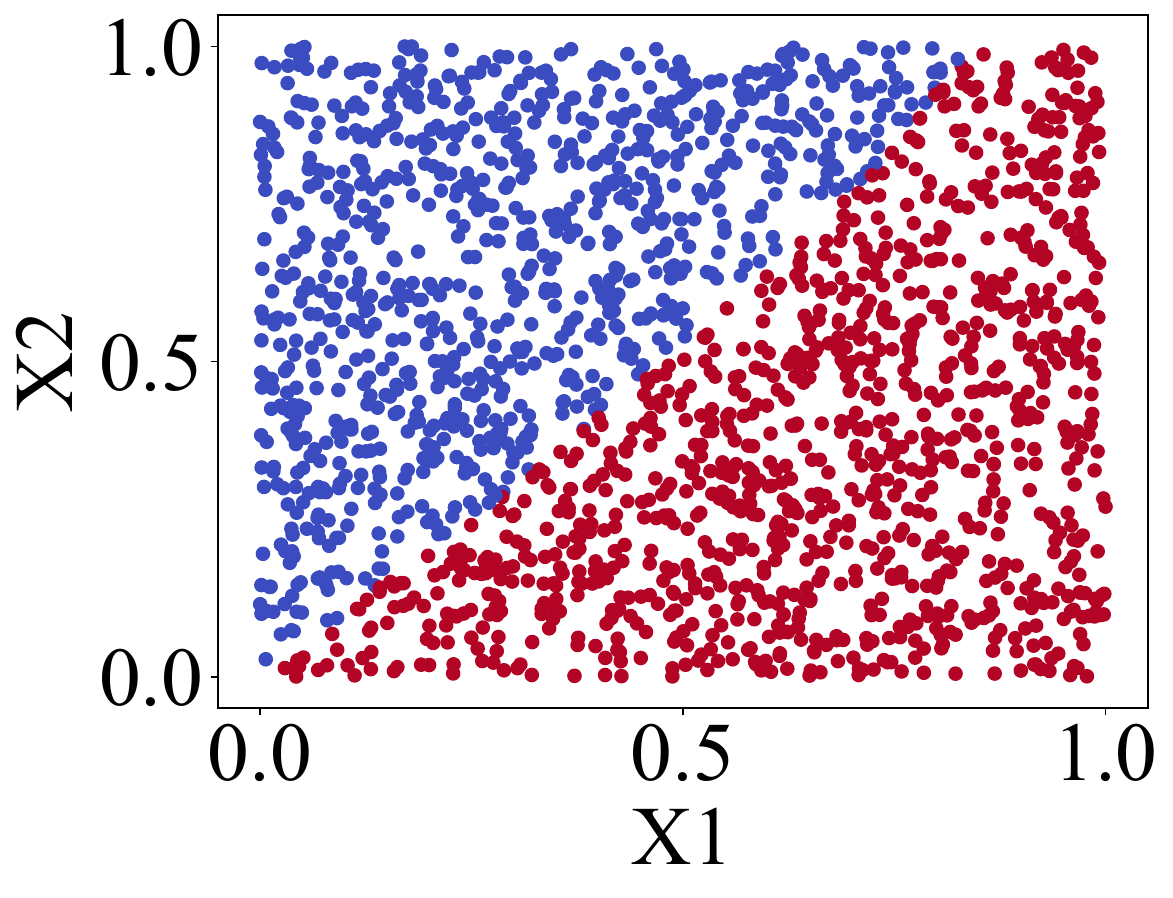}
     \caption{{\sf Case 1 Training Subset}}
     \label{fig:same_concept}
  \end{subfigure}
  \begin{subfigure}{0.32\columnwidth}
     \centering
     \includegraphics[width=\columnwidth]{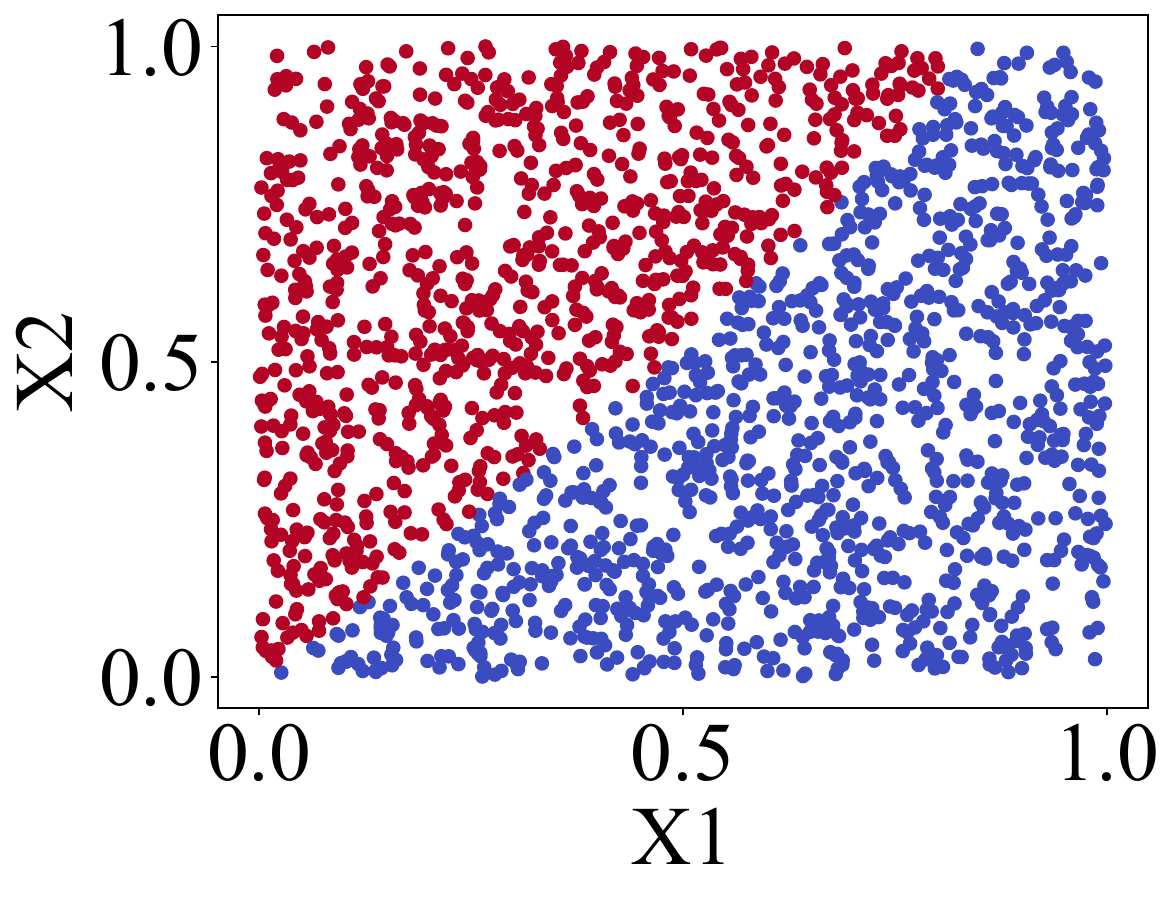}
     \caption{{\sf Case 2 Training Subset}}
     \label{fig:drift_concept}
  \end{subfigure}
  \begin{subfigure}{0.32\columnwidth}
     \centering
     \includegraphics[width=\columnwidth]{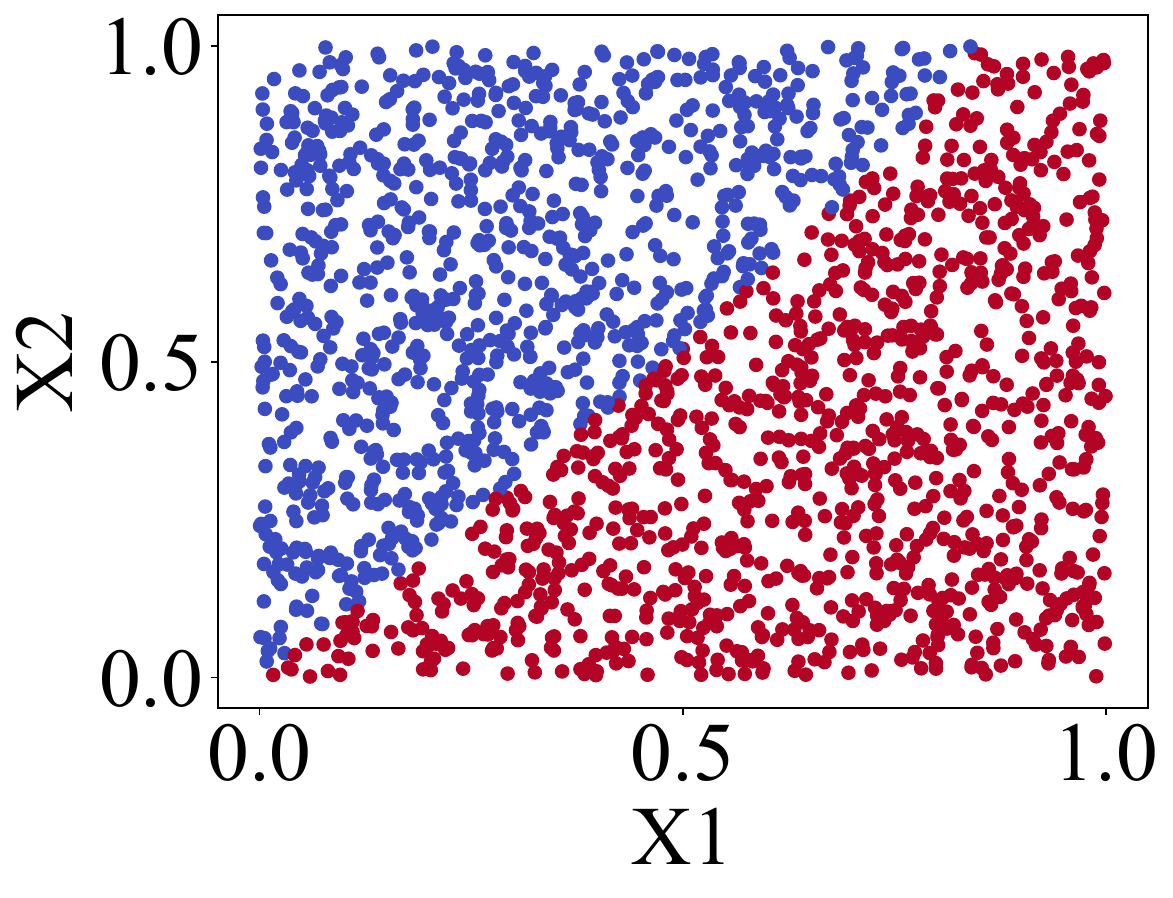}
     \caption{{\sf Validation Set}}
     \label{fig:validation_concept}
  \end{subfigure}
  \caption{The training and validation sets for the two cases.}
\label{fig:casedata}
\end{figure}

\section{Appendix -- Experiments}

\subsection{Dataset Descriptions}

We provide more details of the four synthetic and five real datasets. Among the four synthetic datasets, we know the drift locations for the SEA and Sine datasets, but not for Hyperplane and Random RBF due to the gradual drifting.

\begin{itemize}[leftmargin=*]
\item {\bf SEA}: We generate the Streaming Ensemble Algorithm (SEA) dataset\,\cite{DBLP:conf/kdd/StreetK01}, which is a standard dataset for simulating sudden concept drifts. We generate samples in a three-dimensional feature space with random numeric values that vary from 0 to 10, where only two of them are relevant to the binary classification task. A sample in a segment is in one of the classes if the sum of the first two features is within a defined threshold for that segment. We also flip 10\% of the labels to add noise. We use four thresholds to simulate four different concepts and generate eight data segments with each concept reoccurring once.
\item {\bf Rotating Hyperplane}\,\cite{DBLP:conf/kdd/HultenSD01}: We view hyperplanes as concepts and vary their orientiations and positions to simulate drifts. A hyperplane is defined by feature weights, and we make the weights drift over time. There are ten relevant features including two drift features, and we set 5\% of label flipping noise to the binary class label. We split the dataset into eight data segments that gradually and incrementally drift. 
\item {\bf Random RBF}\,\cite{DBLP:journals/jmlr/BifetHKP10}: We use Random Radial Basis Function to make a number of random centroids and new samples are generated by selecting the center of centroids. We set 50 centroids and make all of them move with the speed of 0.0001 to simulate sudden, gradual, and incremental drifts. There are ten features and five class labels. We split the dataset every 2,000 samples and generate a total of eight data segments with drifts.
\item {\bf Sine}\,\cite{DBLP:conf/sbia/GamaMCR04}: We use four numerical features with values that range from 0 to 1. Two of the features are relevant to a given binary classification task, while the two other features simulate noise. We use four different sine functions as concepts and generate eight data segments with each concept reoccurring once.
\item {\bf Electricity Market}\,\cite{Harries99splice-2comparative}: A real dataset for the Australian New South Wales Electricity Market from 1996 to 1998. Each sample is generated every 30 minutes, so 48 samples of the dataset correspond to one day. There are six features including the NSW electricity demand, the Vic electricity demand, and the scheduled electricity transfer between states. The label identifies the change of the electricity price relative to a moving average of the last 24 hours. As the price is correlated to season, we divide the data into ten 90-day segments.
\item {\bf NOAA Weather}\,\cite{DBLP:journals/tnn/ElwellP11}: This real dataset measures the weather in Bellevue NE during the period of 1949--1999. There are eight features including temperature, dew point, sea level pressure, visibility, and wind speed. The labels indicate whether it rained (positive label) or not (negative label). Since the truth label of drift time is unknown, we consider yearly drift and divide the dataset into ten 5-year segments. Although there are possible drifts between the period, they are mild.
\item {\bf Spam}\,\cite{DBLP:journals/kais/KatakisTV10}: This real dataset consists of email messages from the Spam Assassin Collection. There are 9,324 samples of messages and a message is represented by 499 features of boolean bag-of-words. The labels denote whether a message is spam or legitimate. The characteristics of spam messages gradually change over time which shows gradual concept drift.
\item {\bf Usenet1 and Usenet2}\,\cite{DBLP:conf/ecai/KatakisTV08}: Two real datasets are based on the 20 newsgroup collection with three topics: medicine, space, and baseball. Each sample contains messages about different topics, and a user labels them sequentially by personal interests whether the topic of a message is interesting (1) or junk (0). The user's interests suddenly change for every 300 samples. There are 99 features representing words that appear in all 1,500 messages. There are five segments in each dataset. There are two different reoccurring concepts in Usenet1 and three different reoccurring concepts in Usenet2.
\end{itemize}

\subsection{More Details on Experimental Settings}

In the experiments, we use a simple neural network for two reasons: (1) to be consistent with the baseline methods that also use simple models like decision trees, random forests, and simple neural networks and (2) our method with a simple model is already accurate enough against concept drifts. Our method can certainly be used with more complex models for better performance. 

During model training, if a concept drift occurs, we randomly initialize the model before training it again. While we could also re-use pre-trained models in previous steps, we do not see a clear performance difference between using randomly-initialized and pre-trained models as shown in Table~\ref{tbl:performance_initialization}. For the datasets with low drift severity (SEA, Electricity, Weather, and Spam), using pre-trained models reduces runtime for similar accuracies, but for those with high drift severity (Random RBF and Sine), random initialization gives better accuracy results. Since we use a (small) validation set to select data segments, the concern about initial model randomness is not a critical issue.

\begin{table}[h]
  \setlength{\tabcolsep}{4pt}
  \caption{Accuracy and runtime (sec) results of Quilt with different model initialization methods on the six datasets.} 
  \centering
  \begin{tabular}{l|cccccccccccc}
  \toprule
    {Methods} & \multicolumn{2}{c}{\sf SEA} & \multicolumn{2}{c}{\sf Random RBF} & \multicolumn{2}{c}{\sf Sine} & \multicolumn{2}{c}{\sf Electricity} & \multicolumn{2}{c}{\sf Weather} & \multicolumn{2}{c}{\sf Spam}  \\
    \midrule  
    & {Acc.} & {Time} & {Acc.} & {Time} & {Acc.} & {Time} & {Acc.} & {Time} & {Acc.} & {Time} & {Acc.} & {Time} \\
    \midrule
    {Random} & {.888}\tiny{$\pm$.004} & {2.20} & {.833}\tiny{$\pm$.008} & {3.22} & {.936}\tiny{$\pm$.005} & {4.88} & {.728}\tiny{$\pm$.007} & {5.24} & {.796}\tiny{$\pm$.004} & {2.00} & {.974}\tiny{$\pm$.003} & {2.59}  \\
    {Previous} & {.888}\tiny{$\pm$.004} & {2.04} & {.816}\tiny{$\pm$.015} & {2.28} & {.923}\tiny{$\pm$.007} & {6.43} & {.727}\tiny{$\pm$.007} & {3.90} & {.799}\tiny{$\pm$.005} & {1.81} & {.975}\tiny{$\pm$.003} & {1.02} \\
    \bottomrule
  \end{tabular}
  \label{tbl:performance_initialization}
\end{table}

Regarding the disparity threshold $T_d$, we set it using Bayesian optimization with a search interval between (0, 2). More specifically, for each new data segment, we run random exploration and Bayesian optimization 3 times to set $T_d$. Searching beyond $T_d = 2$ is unnecessary as we empirically observe that it roughly equates to more than 80\% of the labels being drifted, meaning the data segment is useless.

\subsection{More Results on Accuracy, F1 Score, and Runtime}

We compare \method{} with the four types of baselines w.r.t. accuracy, $F_1$ score, and runtime on the four synthetic datasets and five real datasets in Table~\ref{tbl:performance_synthetic} and Table~\ref{tbl:performance_real}, respectively. The observations are the same as in Table~\ref{tbl:performance_representative}.

\begin{table}[h]
  \setlength{\tabcolsep}{7pt}
  \caption{Accuracy, $F_1$ score, and runtime (sec) results on the four synthetic datasets. We compare \method{} with all the four types of baselines.} 
  \centering
  \begin{tabular}{lcccccccccccc}
  \toprule
    {Methods} & \multicolumn{3}{c}{\sf SEA} & \multicolumn{3}{c}{\sf Hyperplane} & \multicolumn{3}{c}{\sf Random RBF} & \multicolumn{3}{c}{\sf Sine}  \\
    \cmidrule{1-13}
    {} & {Acc.} & {F1.} & {Time} & {Acc.} & {F1.} & {Time} & {Acc.} & {F1.} & {Time} & {Acc.} & {F1.} & {Time} \\
    \midrule
    {Full Data} & {.849} & {.881} & {3.36} & {.843} & {.844} & {2.41} & {.821} & {.820} & {9.43} & {.449} & {.436} & {2.25} \\
    {Current Seg.} & {.864} & {.888} & {0.20} & {.893} & {.894} & {0.46} & {.679} & {.673} & {0.56} & {.899} & {.898} & {0.94} \\
    \cmidrule{1-13}
    {HAT} & {.825} & {.862} & {1.38} & {.860} & {.862} & {2.23} & {.514} & {.519} & {2.35} & {.293} & {.305} & {1.67} \\
    {ARF} & {.825} & {.863} & {23.49} & {.797} & {.793} & {26.86} & {.645} & {.642} & {44.64} & {.821} & {.823} & {21.40} \\
    {Learn++.NSE} & {.804} & {.836} & {6.65} & {.750} & {.755} & {7.05} & {.611} & {.612} & {5.68} & {.925} & {.925} & {5.73} \\
    {SEGA} & {.797} & {.842} & {4.37} & {.853} & {.851} & {4.49} & {.825} & {.825} & {4.47} & {.253} & {.260} & {4.35} \\
    \cmidrule{1-13}
    {CVDTE} & {.806} & {.810} & {0.02} & {.744} & {.752} & {0.04} & {.614} & {.621} & {0.05} & {.857} & {.835} & {0.02} \\
    \cmidrule{1-13}
    {GLISTER} & {.857} & {.885} & {25.89} & {.906} & {.905} & {21.64} & {.794} & {.794} & {63.73} & {.879} & {.876} & {14.93} \\
    {GRAD-MATCH} & {.853} & {.884} & {2.13} & {.843} & {.845} & {1.40} & {.790} & {.790} & {6.66} & {.547} & {.529} & {0.80} \\
    \cmidrule{1-13}
    {\bf \method{}} & {$\mathbf{.888}$} & {$\mathbf{.909}$} & {2.20} & {$\mathbf{.911}$} & {$\mathbf{.912}$} & {2.08} & {$\mathbf{.833}$} & {$\mathbf{.833}$} & {3.22} & {$\mathbf{.936}$} & {$\mathbf{.936}$} & {4.88} \\
    \bottomrule
  \end{tabular}
  \label{tbl:performance_synthetic}
\end{table}

\begin{table}[h]
  \setlength{\tabcolsep}{4pt}
  \caption{Accuracy, $F_1$ score, and runtime (sec) results on the five real datasets. We compare \method{} with all the four types of baselines.}
  \centering
  \begin{tabular}{lccccccccccccccc}
  \toprule
    {Methods} & \multicolumn{3}{c}{\sf Electricity} & \multicolumn{3}{c}{\sf Weather} & \multicolumn{3}{c}{\sf Spam} & \multicolumn{3}{c}{\sf Usenet1} & \multicolumn{3}{c}{\sf Usenet2} \\
    \cmidrule{1-16}
    {} & {Acc.} & {F1.} & {Time} & {Acc.} & {F1.} & {Time} & {Acc.} & {F1.} & {Time} & {Acc.} & {F1.} & {Time} & {Acc.} & {F1.} & {Time} \\
    \midrule
    {Full Data} & {.694} & {.758} & {7.42} & {$\mathbf{.800}$} & {.641} & {4.33} & {.970} & {.973} & {1.17} & {.576} & {.512} & {0.23} & {.701} & {.413} & {0.18} \\
    {Current Seg.} & {.709} & {.756} & {0.52} & {.756} & {.509} & {0.26} & {.955} & {.963} & {0.16} & {.752} & {.716} & {0.16} & {.745} & {.613} & {0.18} \\
    \cmidrule{1-16}
    {HAT} & {.691} & {.743} & {6.43} & {.729} & {.452} & {2.10} & {.888} & {.847} & {25.67} & {.622} & {.558} & {0.87} & {.730} & {.472} & {0.87} \\
    {ARF} & {.713} & {.762} & {57.36} & {.775} & {.542} & {30.51} & {.921} & {.931} & {44.83} & {.629} & {.616} & {4.12} & {.682} & {.311} & {4.04} \\
    {Learn++.NSE} & {.698} & {.734} & {17.26} & {.703} & {.523} & {7.86} & {.928} & {.942} & {3.81} & {.433} & {.412} & {0.35} & {.637} & {.251} & {0.33} \\
    {SEGA} & {.637} & {.697} & {10.26} & {.777} & {.602} & {4.11} & {.858} & {.851} & {6.67} & {.403} & {.318} & {0.84} & {.630} & {.207} & {0.84} \\
    \cmidrule{1-16}
    {CVDTE} & {.689} & {.736} & {0.04} & {.731} & {.497} & {0.03} & {.917} & {.918} & {0.12} & {.718} & {.624} & {0.01} & {.689} & {.523} & {0.01} \\
    \cmidrule{1-16}
    {GLISTER} & {.698} & {.741} & {77.46} & {.793} & {$\mathbf{.649}$} & {40.79} & {.971} & {.974} & {14.52} & {.771} & {.736} & {2.07} & {.744} & {.603} & {1.89} \\
    {GRAD-MATCH} & {.686} & {.748} & {5.97} & {.795} & {.622} & {3.51} & {.968} & {.972} & {1.13} & {.630} & {.591} & {0.13} & {.679} & {.420} & {0.12} \\
    \cmidrule{1-16}
    {\bf \method{}} & {$\mathbf{.728}$} & {$\mathbf{.775}$} & {5.24} & {.796} & {.635} & {2.00} & {$\mathbf{.974}$} & {$\mathbf{.976}$} & {2.59} & {$\mathbf{.812}$} & {$\mathbf{.782}$} & {0.86} & {$\mathbf{.771}$} & {$\mathbf{.640}$} & {1.03} \\
    \bottomrule
  \end{tabular}
  \label{tbl:performance_real}
\end{table}

\subsection{More Results on Data Segment Selection Analysis}

We show more average precision and recall results when comparing \method{} to the {\em Best Segments} results on the five real datasets in Table~\ref{tbl:DADSS_table_appendix}. The observations are similar to Table~\ref{tbl:DADSS_table} where the recall is near to 1, and the precision is lower than 1 because of the additional similar data segments selected by \method{}. In addition, Figure~\ref{fig:selection_analysis_appendix} shows the accumulative model evaluation results against incoming data segments on the five real datasets. The observations are also similar to Figure~\ref{fig:selection_analysis} where \method{} sometimes performs even better than {\em Best Segments} because of the better generalization with the extra similar data segments.

\begin{table}[H]
  \setlength{\tabcolsep}{5.2pt}
  \caption{Comparison of \method{}'s selected data segments against the Best Segments results on the real datasets.}
  \centering
  \begin{tabular}{cccccccccc}
    \toprule
    {Metrics} & {\sf Electricity} & {\sf Weather} & {\sf Spam} & {\sf Usenet1} & {\sf Usenet2} \\
    \midrule
    Precision & 0.64 & 0.78 & 0.81 & 0.91 & 0.92 \\
    Recall & 0.95 & 0.90 & 0.93 & 1.00 & 1.00 \\
    \bottomrule
  \end{tabular}
  \label{tbl:DADSS_table_appendix}
\end{table}

\begin{figure}[H]
  \centering
  \begin{subfigure}{0.316\columnwidth}
     \centering
     \includegraphics[width=\columnwidth]{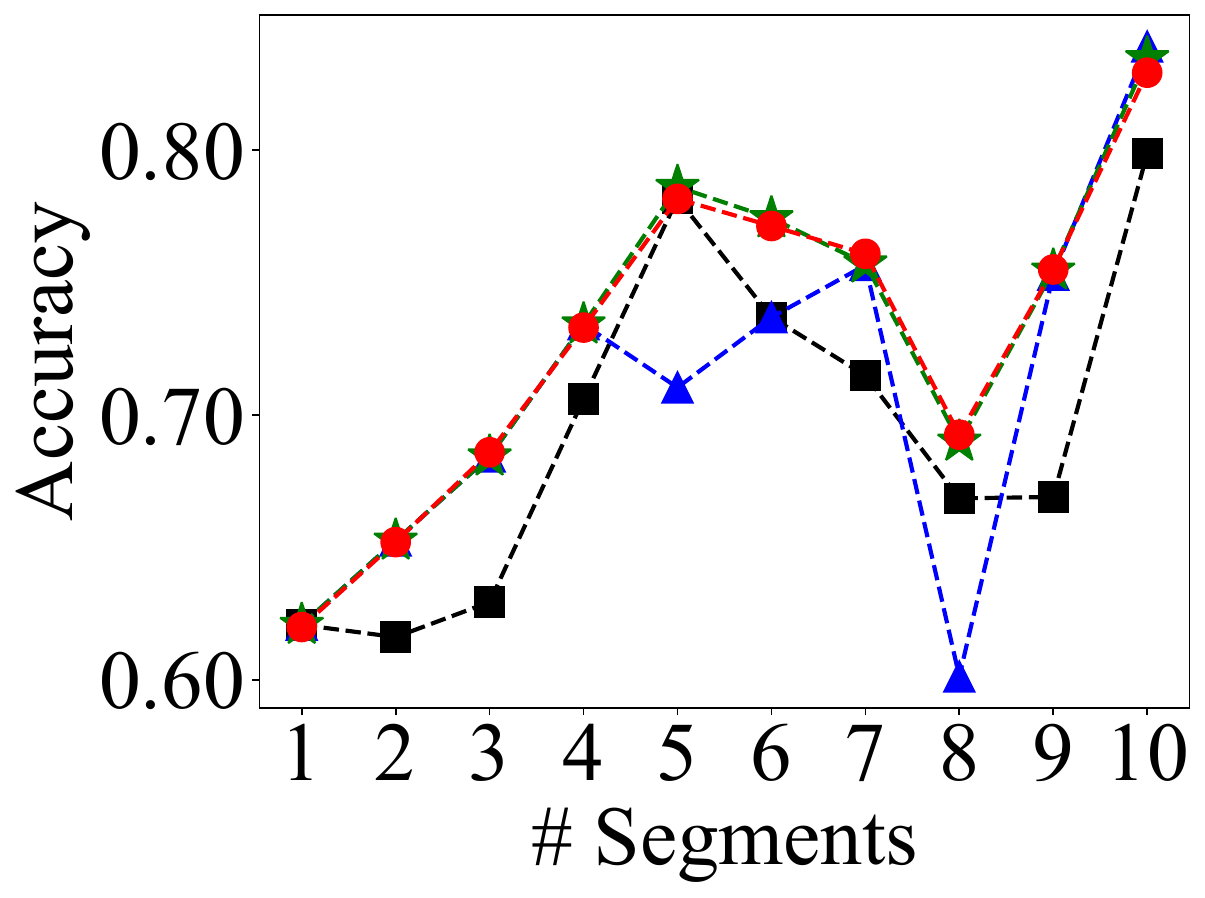}
     \caption{{\sf Electricity}}
     \label{fig:Electricity}
 \end{subfigure}
 \begin{subfigure}{0.316\columnwidth}
     \centering
     \includegraphics[width=\columnwidth]{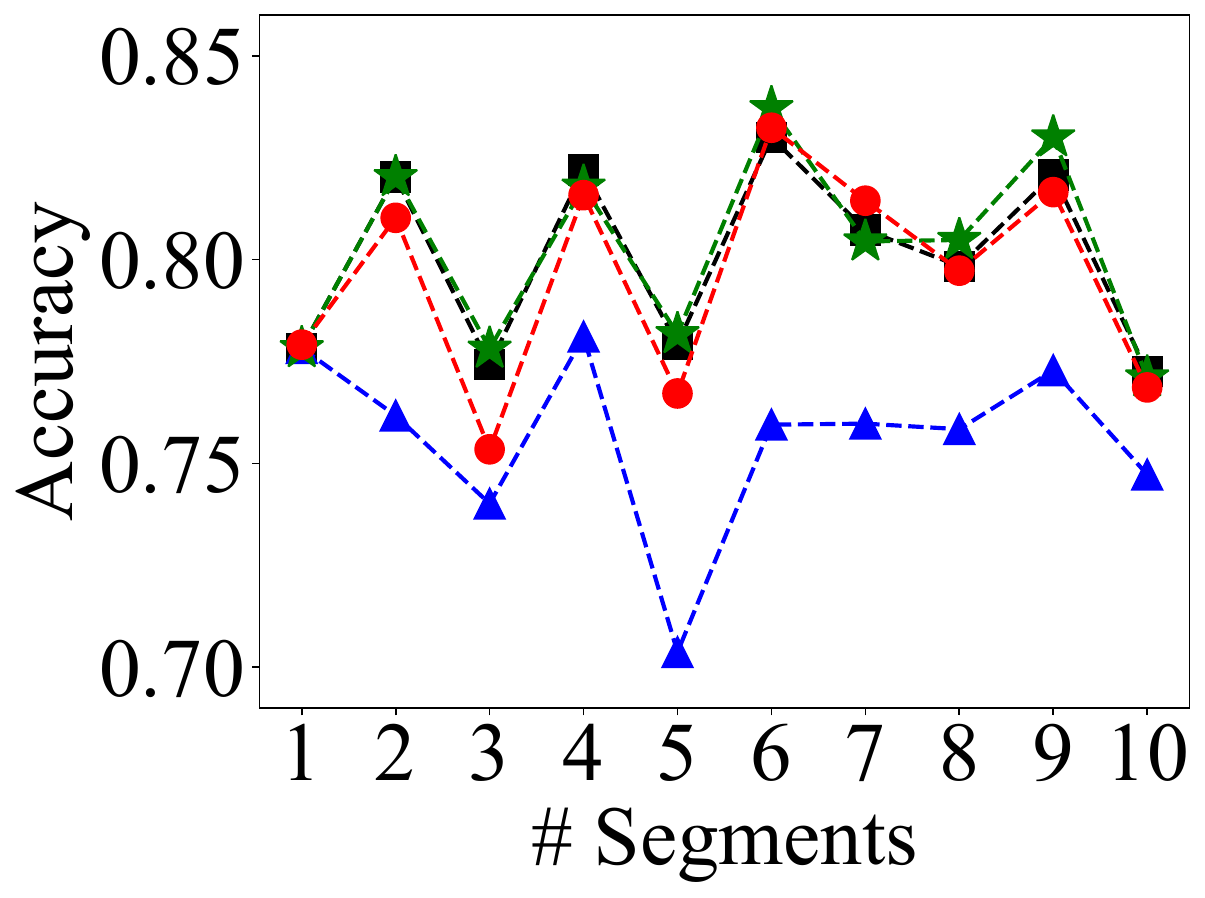}
     \caption{{\sf Weather}}
     \label{fig:Weather}
 \end{subfigure}
 \begin{subfigure}{0.316\columnwidth}
     \centering
     \includegraphics[width=\columnwidth]{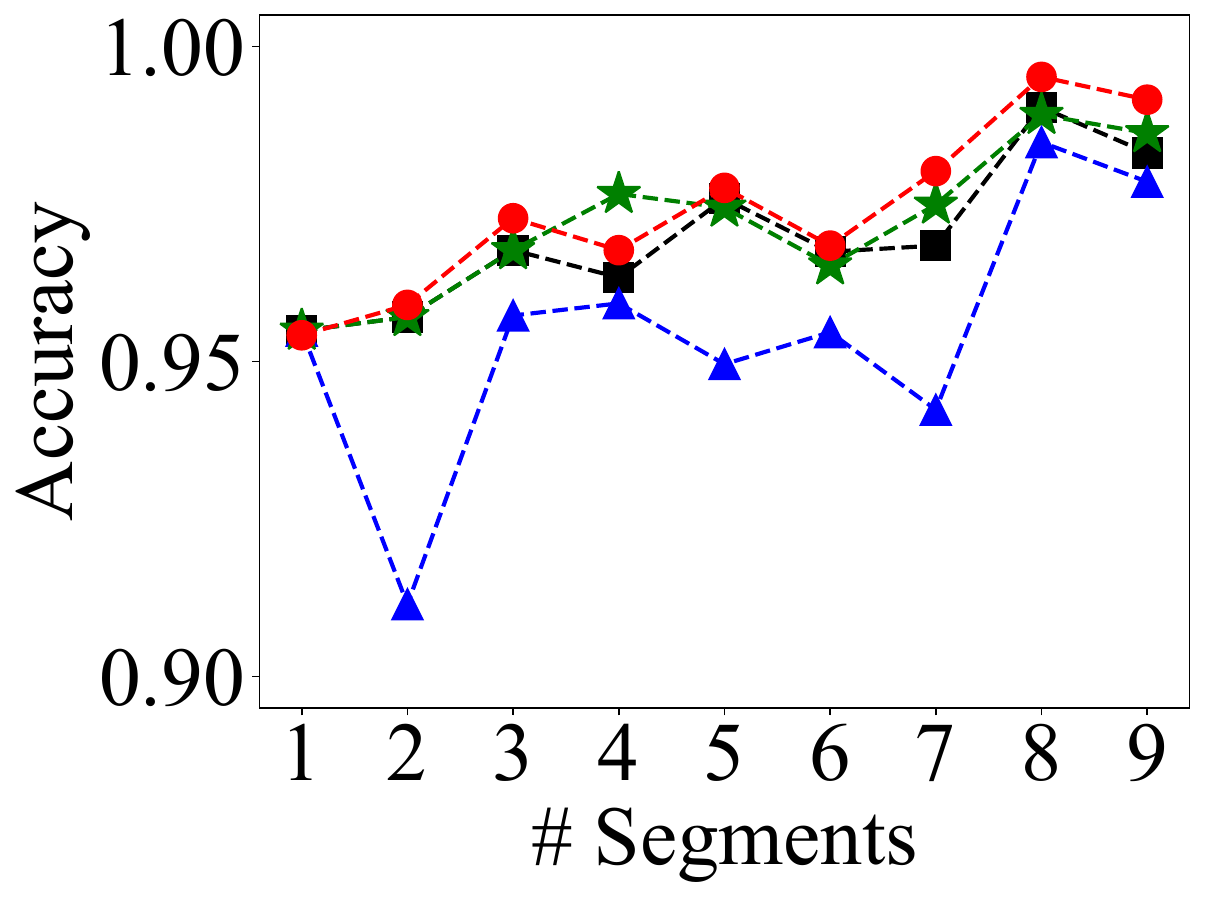}
     \caption{{\sf Spam}}
     \label{fig:Spam}
 \end{subfigure}
 \begin{subfigure}{0.316\columnwidth}
     \centering
     \includegraphics[width=\columnwidth]{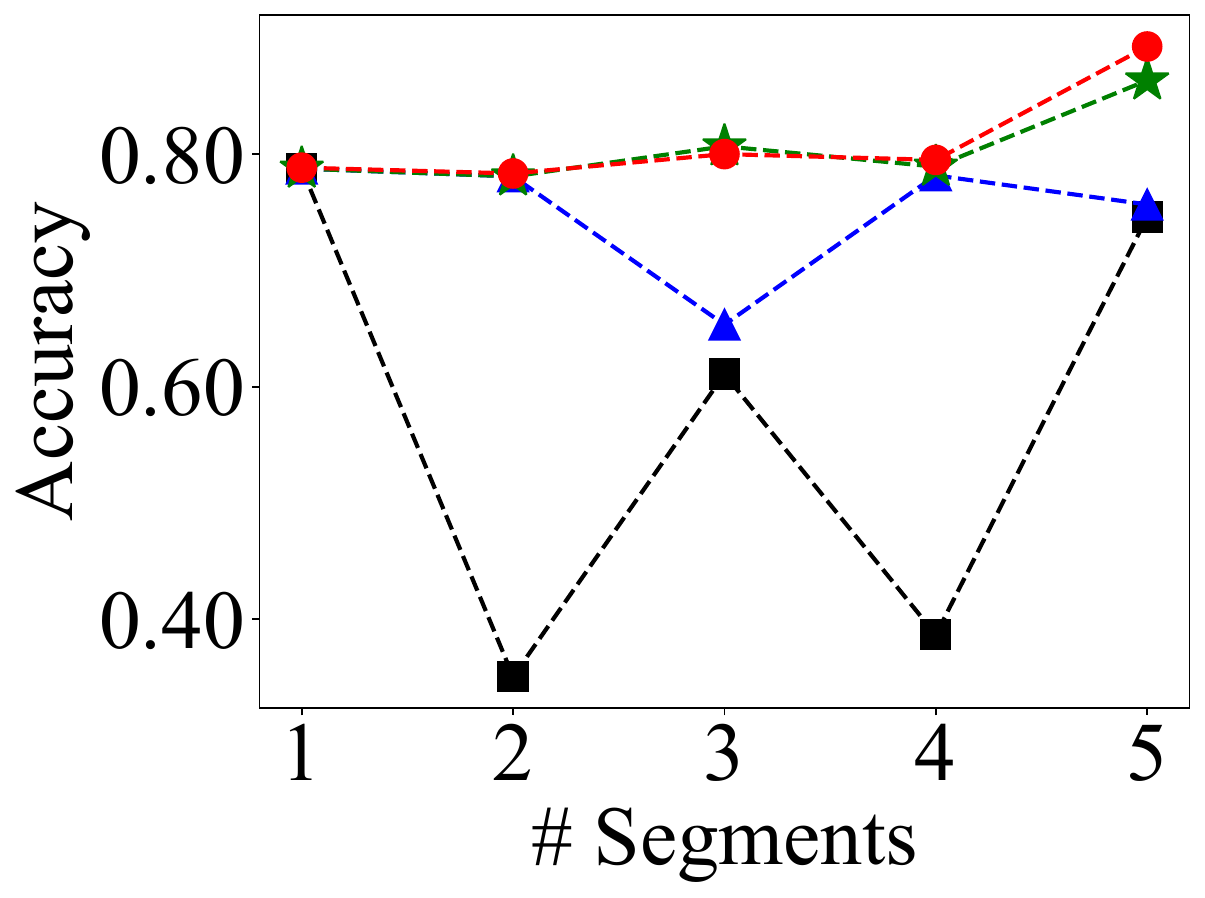}
     \caption{{\sf Usenet1}}
     \label{fig:Usenet1}
 \end{subfigure}
 \begin{subfigure}{0.316\columnwidth}
     \centering
     \includegraphics[width=\columnwidth]{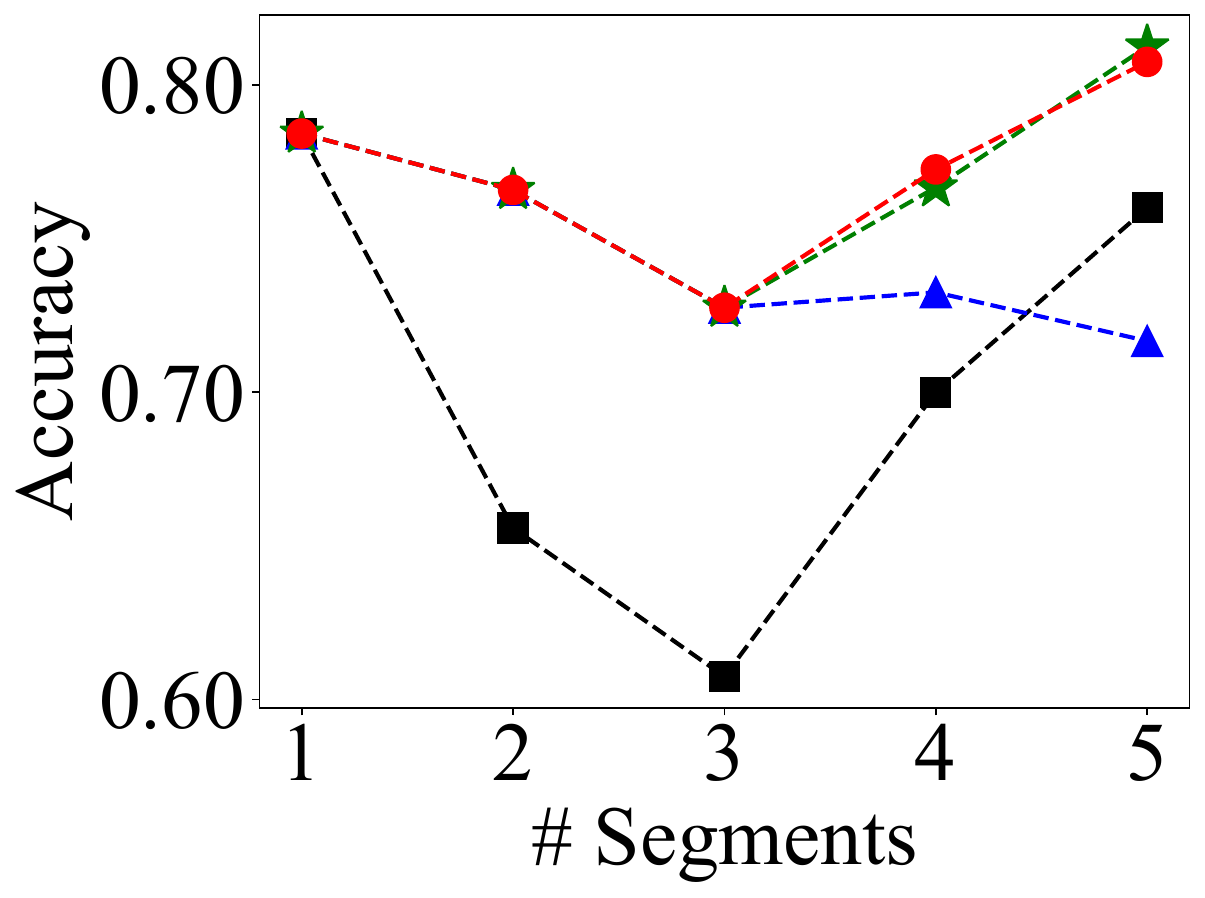}
     \caption{{\sf Usenet2}}
     \label{fig:Usenet2}
 \end{subfigure}
 \begin{subfigure}{0.316\columnwidth}
     \centering
     \includegraphics[width=\columnwidth]{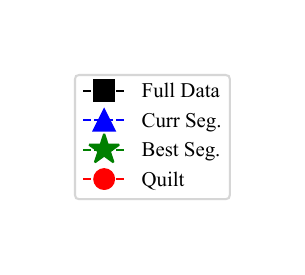}
     \label{fig:legend2}
 \end{subfigure}
 \caption{Accumulative model evaluation results against incoming data segments trained with different data segment selection methods on the five real datasets.}
 \label{fig:selection_analysis_appendix}
\end{figure}

\subsection{More Results on Ablation Study}

We show more results of the ablation study of \method{} to see how the two gradient-based scores contribute to the overall performance. Table~\ref{tbl:variations_all} evaluates \method{} variants when not using one or both scores for all the datasets. We compare the accuracy, $F_1$ score, runtime, the relative speedup compared to when not using the scores, the portion of data segments selected (Usage), and the number of epochs. As a result, the observations are similar to Table~\ref{tbl:variations_six}.

\begin{table}[h]
  \caption{Accuracy, $F_1$ score, runtime (sec), speedup, data segment usage results, and the number of epochs of \method{} when not using one or both scores on all the datasets.} 
  \centering
  \begin{tabular}{clccc|ccc}
  \toprule
    {Datasets} & {Methods} & {Acc.} & {F1.} & {Time (Speed Up)} & {Usage} & {Epochs}\\
    \midrule
    \multirow{4}{*}{\makecell{\sf SEA}} & {W/o both} & {.850\tiny{$\pm$.005}} & {.881\tiny{$\pm$.004}} & {4.49 (1.0$\times$)} & {100.0\%} & {13.9} \\ 
    & {W/o $\mathcal{D}$} & {.890\tiny{$\pm$.004}} & {.911\tiny{$\pm$.003}} & {2.57 (1.7$\times$)} & {42.2\%} & {27.0} \\
    & {W/o $\mathcal{G}$} & {.881\tiny{$\pm$.007}} & {.905\tiny{$\pm$.006}} & {3.06 (1.5$\times$)} & {49.6\%} & {37.0} \\ 
    & {\method{}} & {.888\tiny{$\pm$.004}} & {.909\tiny{$\pm$.003}} & {2.20 (2.0$\times$)} & {30.8\%} & {44.1} \\
    \midrule
    \multirow{4}{*}{\makecell{\sf Hyperplane}} & {W/o both} & {.844\tiny{$\pm$.005}} & {.845\tiny{$\pm$.005}} & {3.99 (1.0$\times$)} & {100.0\%} & {20.9} \\  
    & {W/o $\mathcal{D}$} & {.915\tiny{$\pm$.005}} & {.915\tiny{$\pm$.005}} & {2.39 (1.7$\times$)} & {42.3\%} & {39.9} \\
    & {W/o $\mathcal{G}$} & {.905\tiny{$\pm$.003}} & {.906\tiny{$\pm$.003}} & {2.66 (1.5$\times$)} & {35.3\%} & {75.2} \\ 
    & {\method{}} & {.911\tiny{$\pm$.007}} & {.912\tiny{$\pm$.007}} & {2.08 (1.9$\times$)} & {29.8\%} & {64.6} \\
    \midrule
    \multirow{4}{*}{\makecell{\sf Random RBF}} & {W/o both} & {.822\tiny{$\pm$.007}} & {.822\tiny{$\pm$.007}} & {16.51 (1.0$\times$)} & {100.0\%} & {86.2} \\  
    & {W/o $\mathcal{D}$} & {.828\tiny{$\pm$.008}} & {.828\tiny{$\pm$.008}} & {4.32 (3.8$\times$)} & {45.3\%} & {72.2} \\
    & {W/o $\mathcal{G}$} & {.829\tiny{$\pm$.011}} & {.829\tiny{$\pm$.011}} & {9.69 (1.7$\times$)} & {79.7\%} & {79.0} \\ 
    & {\method{}} & {.833\tiny{$\pm$.008}} & {.833\tiny{$\pm$.008}} & {3.22 (5.1$\times$)} & {39.4\%} & {76.0} \\
    \midrule
    \multirow{4}{*}{\makecell{\sf Sine}} & {W/o both} & {.444\tiny{$\pm$.032}} & {.438\tiny{$\pm$.040}} & {3.43 (1.0$\times$)} & {100.0\%} & {54.4} \\  
    & {W/o $\mathcal{D}$} & {.890\tiny{$\pm$.015}} & {.889\tiny{$\pm$.015}} & {2.75 (1.2$\times$)} & {36.5\%} & {95.2} \\
    & {W/o $\mathcal{G}$} & {.941\tiny{$\pm$.003}} & {.941\tiny{$\pm$.003}} & {7.77 (0.4$\times$)} & {23.2\%} & {245.0} \\ 
    & {\method{}} & {.936\tiny{$\pm$.005}} & {.936\tiny{$\pm$.005}} & {4.88 (0.7$\times$)} & {21.1\%} & {220.4} \\
    \midrule
    \multirow{4}{*}{\makecell{\sf Electricity}} & {W/o both} & {.696\tiny{$\pm$.009}} & {.758\tiny{$\pm$.008}} & {10.71 (1.0$\times$)} & {100.0\%} & {8.6} \\  
    & {W/o $\mathcal{D}$} & {.711\tiny{$\pm$.013}} & {.755\tiny{$\pm$.014}} & {6.57 (1.6$\times$)} & {52.9\%} & {23.1} \\
    & {W/o $\mathcal{G}$} & {.723\tiny{$\pm$.009}} & {.774\tiny{$\pm$.007}} & {6.68 (1.6$\times$)} & {39.6\%} & {41.0} \\ 
    & {\method{}} & {.728\tiny{$\pm$.007}} & {.775\tiny{$\pm$.007}} & {5.24 (2.0$\times$)} & {27.9\%} & {46.4} \\
    \midrule
    \multirow{4}{*}{\makecell{\sf Weather}} & {W/o both} & {.798\tiny{$\pm$.006}} & {.631\tiny{$\pm$.020}} & {7.67 (1.0$\times$)} & {100.0\%} & {24.1} \\  
    & {W/o $\mathcal{D}$} & {.794\tiny{$\pm$.004}} & {.629\tiny{$\pm$.014}} & {2.19 (3.5$\times$)} & {40.9\%} & {26.4} \\
    & {W/o $\mathcal{G}$} & {.800\tiny{$\pm$.006}} & {.643\tiny{$\pm$.020}} & {6.85 (1.1$\times$)} & {75.5\%} & {32.8} \\ 
    & {\method{}} & {.796\tiny{$\pm$.004}} & {.635\tiny{$\pm$.012}} & {2.00 (3.8$\times$)} & {30.7\%} & {37.6} \\
    \midrule
    \multirow{4}{*}{\makecell{\sf Spam}} & {W/o both} & {.970\tiny{$\pm$.003}} & {.973\tiny{$\pm$.002}} & {4.51 (1.0$\times$)} & {100.0\%} & {12.6} \\  
    & {W/o $\mathcal{D}$} & {.973\tiny{$\pm$.004}} & {.975\tiny{$\pm$.003}} & {2.31 (2.0$\times$)} & {52.1\%} & {24.6} \\
    & {W/o $\mathcal{G}$} & {.972\tiny{$\pm$.002}} & {.975\tiny{$\pm$.002}} & {4.02 (1.1$\times$)} & {60.5\%} & {20.4} \\ 
    & {\method{}} & {.974\tiny{$\pm$.003}} & {.976\tiny{$\pm$.002}} & {2.59 (1.7$\times$)} & {39.7\%} & {31.4} \\
    \midrule
    \multirow{4}{*}{\makecell{\sf Usenet1}} & {W/o both} & {.583\tiny{$\pm$.043}} & {.522\tiny{$\pm$.074}} & {0.41 (1.0$\times$)} & {100.0\%} & {9.1} \\  
    & {W/o $\mathcal{D}$} & {.820\tiny{$\pm$.020}} & {.786\tiny{$\pm$.027}} & {0.42 (1.0$\times$)} & {49.1\%} & {22.4} \\
    & {W/o $\mathcal{G}$} & {.814\tiny{$\pm$.038}} & {.780\tiny{$\pm$.043}} & {0.45 (0.9$\times$)} & {38.6\%} & {30.0} \\ 
    & {\method{}} & {.812\tiny{$\pm$.026}} & {.782\tiny{$\pm$.027}} & {0.51 (0.8$\times$)} & {38.0\%} & {36.9} \\
    \midrule
    \multirow{4}{*}{\makecell{\sf Usenet2}} & {W/o both} & {.703\tiny{$\pm$.015}} & {.418\tiny{$\pm$.060}} & {0.40 (1.0$\times$)} & {100.0\%} & {4.8} \\  
    & {W/o $\mathcal{D}$} & {.754\tiny{$\pm$.026}} & {.573\tiny{$\pm$.040}} & {0.37 (1.1$\times$)} & {41.1\%} & {22.7} \\
    & {W/o $\mathcal{G}$} & {.772\tiny{$\pm$.032}} & {.625\tiny{$\pm$.035}} & {0.60 (0.7$\times$)} & {30.0\%} & {48.8} \\ 
    & {\method{}} & {.771\tiny{$\pm$.031}} & {.640\tiny{$\pm$.035}} & {0.57 (0.7$\times$)} & {28.0\%} & {49.0} \\
  \bottomrule
  \end{tabular}
  \label{tbl:variations_all}
\end{table}

\end{document}